%% file: main.tex
\documentclass[11pt]{article}
\usepackage[margin=1in]{geometry}

\usepackage{authblk}

\usepackage[round]{natbib}

\input{macros}

\author[1]{Sharmila Duppala}
\author[1]{ Juan Luque}
\author[1,2]{John Dickerson}
\author[3]{Seyed A. Esmaeili}

\affil[1]{University of Maryland, College Park}
\affil[2]{Arthur}
\affil[3]{University of Chicago}

\date{}

\begin{document}
\title{Robust Fair Clustering with Group Membership Uncertainty Sets}
\maketitle

\begin{abstract}
  We study the canonical fair clustering problem where each cluster is constrained to have close to population-level representation of each group. Despite significant attention, the salient issue of having incomplete knowledge about the group membership of each point has been superficially addressed. In this paper, we consider a setting where the assigned group memberships are noisy. We introduce a simple noise model that requires a small number of parameters to be given by the decision maker. We then present an algorithm for fair clustering with provable \emph{robustness} guarantees. Our framework enables the decision maker to trade off between the robustness and the clustering quality. Unlike previous work, our algorithms are backed by worst-case theoretical guarantees. Finally, we empirically verify the performance of our algorithm on real world datasets and show its superior performance over existing baselines.   
\end{abstract}
\input{01introduction.tex}
\input{02relatedwork.tex}
\input{03preliminaries.tex}

\input{041robustfaircluster}

\input{04usefulfacts}

\input{05algorithms}

\input{06experiments}

\bibliographystyle{abbrvnat}
\bibliography{references}



\input{07appendix}
\end{document}

%% file: macros.tex
\usepackage{amsmath,amsfonts,amssymb,amsthm,  mathrsfs,mathtools}
\usepackage{algorithmic}
\usepackage{algorithm}
\usepackage{subfigure}
\usepackage{caption}
\usepackage{subcaption}
\usepackage{wrapfig}
\usepackage{authblk}
\usepackage[utf8]{inputenc} 

\usepackage[capitalize,noabbrev]{cleveref}

\usepackage{babel}

\DeclareCaptionLabelFormat{andtable}{#1~#2  \&  \tablename~\thetable}

\usepackage{times}  
\usepackage{helvet}  
\usepackage{courier}  
\usepackage{caption} 
\usepackage{tikz}
\usetikzlibrary{arrows.meta, automata, positioning}
\usepackage{thmtools} 
\usepackage{thm-restate}
\usepackage{empheq}

\newtheorem{fact}{Fact}[section]

\newtheorem{lemma}{Lemma}[section]

\newtheorem{claim}{Claim}[section]

\newcommand{\cP}{\mathcal{P}}

\newcommand{\cU}{\mathcal{U}}
\newcommand{\cH}{\mathcal{H}}

\newcommand{\cI}{\mathcal{I}}

\newcommand{\cih}{C_{i,h}}

\newcommand{\nh}{n_h}

\NewDocumentCommand{\proj}{s e{_} m}{%
    \IfBooleanTF{#1}{#3_{\stretchrel*{\parallel}{\perp}#2}}{\vectorprojection_{#2}#3}
}

\newcommand{\bR}{\mathbb{R}}

\usepackage{xspace}

\usepackage{tabularx}
\usepackage{colortbl}

\newcommand{\rfcalg}{\textsc{RobustAlg}}
\newcommand{\probalg}{\textsc{ProbAlg}}
\newcommand{\detalg}{\textsc{DetAlg}}

\DeclareMathOperator*{\argmin}{arg\,min}

\usepackage{booktabs}
\usepackage{multirow}

\usepackage{dsfont}

\usepackage{xfrac}
\usepackage{enumitem}
\usepackage{lipsum}                     
\usepackage{xargs}                      
\usepackage{xcolor}  
\usepackage[textsize=tiny]{todonotes}

\usepackage{graphics}

\newcommand{\sI}{\mathscr{I}}
\newcommand{\rfc}[1]{\textsc{RobustFair-$k$-Center}}

\newcommand{\fc}[1]{\textsc{Fair-$k$-Center}}
\newcommand{\fa}[1]{\textsc{FairAssignment}}

\newcommand{\lbc}[1]{\textsc{LowerBound}-$(k, p)$-\textsc{Clustering}}
\newcommand{\optlb}[1]{$\text{OPT}_{\text{LB}}(\cI)$}
\newcommand{\optf}[1]{$\text{OPT}_{\text{RF}}$}

\newcommand{\optass}[1]{$\text{OPT}_{\text{Asgn}}(\sI)$}

\newcommand{\rfa}[1]{$T$-\textsc{RobustFairAssignment}}
\newcommand{\algrfa}[1]{\textsc{Algorithm-RobustFairAssignment}}
\newcommand{\trf}[1]{\textit{robust fair}}

\newcommand{\optvll}[1]{$\text{OPT}_{\text{vll}}(\cI)$}

\usepackage{lipsum}                     
\usepackage{xargs}                      
\usepackage{xcolor}  
\usepackage{todonotes}
\newcommandx{\unsure}[2][1=]{\todo[linecolor=red,backgroundcolor=red!25,bordercolor=red,#1]{#2}}
\newcommandx{\change}[2][1=]{\todo[linecolor=blue,backgroundcolor=blue!25,bordercolor=blue,#1]{#2}}
\newcommandx{\info}[2][1=]{\todo[linecolor=OliveGreen,backgroundcolor=OliveGreen!25,bordercolor=OliveGreen,#1]{#2}}
\newcommandx{\improvement}[2][1=]{\todo[linecolor=green,backgroundcolor=green!25,bordercolor=green,#1]{#2}}

\newcommand{\vS}{S_{\text{vll}}}

\newcommand{\phis}{\phi^*}
\newcommand{\mhin}{m_{h}^+}
\newcommand{\mhout}{m_{h}^-}

\newcommand{\outm}{\sum_{h\in \cH}\mhout}

\newcommand{\lp}{\mathnormal{\text{LP}(S,R)}}

\newcommand{\sgood}{\mathnormal{\hat{S}}}

\newcommand{\floor}[1]{\left\lfloor #1 \right\rfloor}
\newcommand{\ceil}[1]{\left\lceil #1 \right\rceil}

\newcommand{\mhp}{m_{h}^{+}}
\newcommand{\mhm}{m_{h}^{-}}

\newcommand{\adult}{\textbf{Adult}\xspace}
\newcommand{\cens}{\textbf{Census1990}\xspace}
\newcommand{\bank}{\textbf{Bank}\xspace}
\newcommand{\diabetes}{\textbf{Diabetes}\xspace}

%% file: 01introduction.tex
\section{Introduction}\label{sec:intro} 
Machine learning and algorithmic-based decision-making systems have seen a remarkable proliferation in the last few decades. These systems are used in financial crime detection \citep{nicholls2021financial,kumar2022exploitation}, loan approval \citep{sheikh2020approach,arun2016loan}, automated hiring systems \citep{mahmoud2019performance,van2021machine}, and recidivism prediction \citep{travaini2022machine,ghasemi2021application}. The clear effect of these applications on the welfare of individuals and groups coupled with recorded instances of algorithmic bias and harm \citep{danks2017algorithmic,panch2019artificial} has made fairness---in its many forms under different interpretations, with its many definitions---a prominent consideration in algorithm design. 


Thus, it is unsurprising that fair \emph{unsupervised learning} has received great interest in the AI/ML, statistics, operations research, and optimization communities---including \emph{fair clustering}.
Clustering is a central problem in AI/ML and operations research and arguably the most fundamental problem in unsupervised learning writ large. The literature in fair clustering has produced a significant number of publications spanning a wide range of fairness notions~\citep[see, e.g.,][for an overview]{fc_tutorial}. However, the most prominent of the fairness notions that were introduced is the \emph{group fairness} notion due to~\citet{chierichetti2017fair}, \citet{bercea2018cost}, and \citet{bera2019fair}. Since our paper is concerned with this notion in particular, for ease of exposition we will simply refer to it as fair clustering. In fair clustering, each point belongs to a demographic group and therefore each demographic group has some percentage representation in the entire dataset.\footnote{As a simple example, the demographic groups could be based on income. Therefore, each point would belong to an income bracket and each income bracket would have some percentage representation (e.g., $20\%$ belong to group ``$\leq \text{USD}\$30\text{k}$,'' $10\%$ belong to group ``$\geq \text{USD}\$250\text{k}$,'' and so on) of the dataset.} Like in agnostic (ordinary or ``unfair'') clustering the dataset is partitioned into a collection of clusters. However, unlike agnostic clustering each cluster must have a proportional representation of each group that is close to the representation in the entire dataset. For example, if the dataset consists of groups $A$ and $B$ at $30\%$ and $70\%$ representation, respectively. Then each cluster in the fair clustering should have a  $(30\pm \epsilon)\%$ and  $(70\pm \epsilon)\%$ representation of groups $A$ and $B$, respectively. 

%

One can see a significant advantage behind group fairness. Each cluster has close to dataset-level representation\footnote{And, ideally, \emph{population-level} representation---yet this may not hold in common machine learning applications, where proportionally sampling an underlying population to form a training dataset requires deep nuance.  For a discussion of biased sampling and its implications in fair machine learning, we direct the reader to~\citet[][Chapters 4 \& 6]{barocas-hardt-narayanan}.} of each group, so any outcome associated with any cluster will affect all groups proportionally, satisfying the disparate impact doctrine \citep{feldman2015certifying}, as discussed by~\citet{chierichetti2017fair}. Despite the attractive properties of group fairness, it requires complete knowledge of each point's membership in a group. In practice---say, in an advertising setting where membership is estimated via a machine learning model, or in a lending scenario where membership may be illegal to estimate at train time---knowledge of group membership may range from noisy, to adversarially corrupted, to completely unknown. This salient problem has received significant attention in fair \emph{classification}~\citep[see, e.g.,][]{awasthi2020equalized,awasthi2021evaluating,wang2020robust,hashimoto2018fairness,kallus2022assessing,lamy2019noise}. However, this important consideration has not received significant attention in fair \emph{clustering} with the exception of the theoretical work of~\citet{esmaeili2020probabilistic} that introduced uncertain group membership and the empirical work of~\citet{Chhabra23:Robust}, who provide a data-driven approach to achieve robustness against adversarial perturbations on fair clustering systems. 

Our paper addresses the practical modeling shortcomings of both ~\citet{esmaeili2020probabilistic} and ~\citet{Chhabra23:Robust} and gives new theoretical worst-case guarantees. In short, the model of \citet{esmaeili2020probabilistic} makes the strong assumption of having probabilistic information about the group membership of each point in the dataset and the weak guarantee of having proportional representation of each group in every cluster but only in expectation. Further, \citet{Chhabra23:Robust} looks into \emph{black-box adversarial perturbations} on fair clustering. However, in their model it is assumed that only a fixed subset of points in the dataset will have their memberships perturbed and it is not clarified how this fixed subset is exactly decided. In Section~\ref{subsec:prior_noise_models} and Appendix~\ref{app:prior_noise} we give a more detailed comparison to these prior works and demonstrate their weaknesses.

\paragraph{Outline and Contributions:} In Section \ref{sec:relatedwork}, we briefly go over some prior work in fair clustering and other works in fair classification with emphasis on papers that tackle the incomplete/noisy group membership case. Then in Section \ref{sec:prelims}, we formally describe the basic clustering setting and introduce our notation then we give an overview of the prior noise models of \citep{esmaeili2020probabilistic} and \citep{Chhabra23:Robust}. In Section~\ref{sec:model}, we present our noise model. Our model requires a small number of parameters as input instead of full probabilistic information for each point. In fact, as a special case it can be given only one parameter that represents the bound on the maximum number of incorrectly assigned group memberships in the dataset. Based on the framework of robust optimization, we then define the \emph{robust fair clustering} problem for the $k$-center objective. In Section \ref{sec:analysis}, we present our theoretically grounded algorithm to solve the robust fair $k$-center problem. Our algorithms require making careful observations about the structure of a robust fair solution and represents a novel addition to the existing fair clustering algorithms. Finally, in Section \ref{sec:experiments} we validate the performance of our algorithm on real world datasets and show that it has superior performance in comparison to the existing methods. 

%% file: 02relatedwork.tex
\section{Additional Related Work}\label{sec:relatedwork}

We will focus on the fairness notion most relevant to us in fair clustering, specifically where the solution is constrained to have proportional group representation in each cluster \citep[e.g.,][and others]{chierichetti2017fair,bercea2018cost,bera2019fair,dickerson2023doubly,wang2023scalable,Zeng_2023_CVPR}. Under the assumption that group memberships are perfectly known, this notion is well-investigated. For example, \citet{backurs2019scalable} gives faster scalable algorithms for this problem to handle large datasets. \citet{bera2019fair} have considered a variant of this problem when each point is allowed to belong to more than one group simultaneously. Further, variants of this notion in non-centroid based clustering have also been considered. \citet{ahmadian2020fair} address the same fairness notion in correlation clustering whereas \citet{kleindessner2019guarantees} address it in spectral clustering, and \citet{knittel2023generalized,Knittel23:Fair} address it in hierarchical clustering. 

The problem of incomplete and imperfect knowledge of group memberships has received significant attention in fair classification. \citet{awasthi2020equalized} study the effects on the equalized odds notion of \citet{hardt2016equality} when the group memberships are perturbed. \citet{awasthi2021evaluating} study the effects of using a classifier to predict the group membership of a point on the bias of downstream ML tasks. \citet{kallus2022assessing} study a similar problem but focusing mainly on assessing the disparate impact in various applications when the group memberships are not unavailable and have to be predicted instead. Robust optimization methods were used to obtain fair classifiers under the setting of noisy group memberships by~\citet[e.g.,][]{wang2020robust} and unavailable group memberships by~\citet[e.g.,][]{hashimoto2018fairness}. While our problem falls under the robust optimization framework, our techniques are very different. 

%% file: 03preliminaries.tex
\section{Preliminaries and Previous Noise Models}\label{sec:prelims}
In this section we go through preliminary background, notation, and previously introduced noise models in clustering.  
\subsection{Preliminaries}
Let $\cP$ be a set of $n$ points in a metric space with distance function $d : \cP \times \cP \rightarrow \bR_{\geq 0}$. 
In $k$-center clustering, the goal is to select a set of centers $S$ from $\cP$ of at most $k$ points and an assignment $\phi : \cP \rightarrow S$ minimizing the clustering cost which is $\text{cost}(S,\phi):= \max_{j \in \cP}d(j,\phi(j))$, i.e., the cost is the maximum distance between a point and its assigned center. Since the clustering cost is the maximum distance between a point and its center, we also refer to that cost as the clustering radius or just radius. Clearly, in the ordinary $k$-center problem, $\phi$ will assign each point $j \in \cP$ to its closest center in $S$, i.e., $\phi(j)=\argmin_{i \in S} d(j,i)$. On the other hand, finding $\phi$ is non-trivial in more general $k$-center variants when constraints are imposed, as $\phi$ may assign points to centers that are further away to satisfy the imposed constraint. 

We index the set of $\ell$ many demographic groups that exist in the dataset by $\cH = \{1,2,\dots, \ell\}$. Following the fair clustering literature we associate a specific color with each group \citep{chierichetti2017fair,bercea2018cost,bera2019fair}. Therefore, we use the words group and color interchangeably. Let $\cP_h$ denote the subset of points in $\cP$ that are assigned color $h$. Each point \( j \in \cP \) belongs to exactly one color from the set of colors \( \cH \). We can equivalently describe this assignment using the function \( \chi: \cP \to \cH \), such that for any \( j \in \cP_h \), the color assignment for \( j \) would be \( \chi(j) = h \). We denote the total number of points of color $h$ by $n_h=|\cP_h|$, it follows that $\sum_{h \in \cH} n_h = n$.
Further, given a solution $(S, \phi$), for each $i\in S$, $C_i$ denotes the set of points assigned to center $i$ (i.e., cluster $i$) and $C_{i,h}$ denotes the subset of points in that cluster belonging to group $h$. 
The \fc{} problem~\citep[e.g.,][]{chierichetti2017fair,bercea2018cost,bera2019fair,esmaeili2020probabilistic} adds the following fairness constraint to the $k$-center objective, formally the optimization problem is:
\begin{subequations} \label{lp:faircluster}
\begin{equation} \label{eq:cluster_obj_fc}
\min\limits_{S : |S|\leq k,  \phi}  \ \ \max_{j \in \cP} \ d(j, \phi(j)) 
\end{equation}    
\begin{equation}  \label{eq:fair}
  \forall i \in S, \forall h \in \cH:  l_h   \hspace{0.1cm} \leq  \hspace{0.1cm} \frac{|\cih|}{|C_i|}  \hspace{0.1cm} \leq   \hspace{0.1cm} u_h   \hspace{0.5cm}  
\end{equation}
\end{subequations}
where $l_h$ and $u_h$ are proportion bounds that satisfy $0  < l_h \leq r_h \leq u_h <  1$ 
with $r_h$ being the ratio (proportion) of group $h$ in the entire set of points, i.e., $r_h\!:= \frac{n_h}{n}$. Therefore, 
an instance of \fc{} is parametrized by the tuple $(\cP,\chi,k,\cH,\Vec{l},\Vec{u})$. 

\subsection{Previous Noise Models in Fair Clustering}\label{subsec:prior_noise_models}
In this section we give more  details about \citep{esmaeili2020probabilistic} and \citep{Chhabra23:Robust}, the two prior works which have considered robustness in fair clustering. \citep{esmaeili2020probabilistic} introduced a probabilistic noise model where each point $j \in \cP$ has a probability $p_{j,h} \in [0,1]$ of belonging to group $h$, with $\sum_{h\in \cH}p_{j,h} = 1$. While their algorithms satisfy proportional fairness constraints in expectation\footnote{Since each point has some probability of belonging to each specific group, one can calculate the expected number of points belonging to a specific group in a clustering by simply adding the points' probabilities in that cluster.}, the worst-case realization can significantly violate these constraints as noted earlier. In fact, in \cref{app:pfc_discussion} we show an example where a clustering of the given points satisfies fairness in expectation, but violates it completely in realization. This highlights a core deficiency in this model. 


\citet{Chhabra23:Robust} introduced an adversarial model where the adversary has access to a subset of points whose group memberships can be modified. However, they do not specify how this subset is selected. In their experiments, they independently sample points with equal probability and add them to this subset. We can construct instances where with high probability certain point combinations are never sampled in the subset, thereby heavily restricting the model’s capability. 
We give a concrete discussion of this in \cref{app:chhabra_discussion}. Moreover, their algorithm does not have theoretical guarantees. 




%% file: 041robustfaircluster.tex
\section{Our Noise Model and Problem Statement}\label{sec:model}
In our model we assume that there exists a number of points whose group memberships (colors) have been incorrectly assigned to other groups. The two main considerations in our model are that: (1) in general these incorrect assignments exhibit a heterogeneity across the groups and (2) that the incorrect assignments can be arbitrarily allocated across the dataset. 

The first consideration is based on the fact that in many settings there exist group memberships that are more desirable than others and therefore individuals may misreport their group memberships as other more favorable groups \citep{krumpal2013determinants}. Further, the mechanism through which the group memberships were assigned may exhibit higher error rates for particular groups. For example, the method used to elicit group memberships may fail with higher rates on some particular groups. Therefore, noise exhibits heterogeneity across the groups and an effective noise model should capture that.

The second consideration is based on the fact that incorrect group assignments could arise from a set of possibilities and therefore unlike \cite{Chhabra23:Robust} we should not assume knowledge of these particular noisy points. The (noise) perturbations in the group assignments could have resulted from a process similar to iid noise as done in \cite{Mehrotra22:Fair} and \cite{mehrotra2021mitigating}. At another extreme, the group memberships could have been assigned using a machine learning classifier which predicts the group memberships, in that case if the classifier's errors are localized to a specific region in the feature space\footnote{This could be the case, if the training dataset happens to be particularly scare in that region.} then clearly the group membership perturbations do not act similar to random noise. 
Further, note that both scenarios are empirically well-motivated and could possibly occur in the same dataset simultaneously. Therefore, an effective noise model should not assume knowledge of the spatial noise distribution and allow noise to be arbitrarily allocated across the dataset. 


Now, we delve into the formal description of the model. For a given color $h \in \cH$ we associate two parameters $\mhp$ and $\mhm$, the first is the maximum number of points that were mistakenly assigned group memberships other than $h$ and the second is the maximum number of points that were mistakenly assigned to group $h$.

Further, for a given set of value $\mhp$ and $\mhm$, by definition the consistency of the values requires that the following inequalities should be satisfied:
\begin{align}
    & \forall h \in \cH: \mhp \leq \sum_{g \in \cH, g \neq h} m_g^{-} \label{eq:plus} \\     
    & \forall h \in \cH: \mhm \leq \sum_{g \in \cH, g \neq h} m_g^{+} \label{eq:minus}
\end{align}
In words, the first inequality \eqref{eq:plus} simply states that no color should ``gain'' more points than the total number of points ``lost'' by the other groups. Similarly, the second inequality \eqref{eq:minus} states that if a color loses some number of points than the rest of the colors must gain at least the same amount in total. Note that since no color can lose more points than it has, an additional set of inequalities is also implied, namely $\forall h \in \cH: \mhm \leq n_h$. We did not list it as we assume that any given set values of $\mhm$ always satisfy it.  

The values of $\mhp$ and $\mhm$ lead to new  possible group membership assignments (colorings) $\hat{\chi}$ other than the original coloring $\chi$. Following the language of robust optimization \citep{ben2009robust}, the set of all possible colorings $\hat{\chi}: \cP \to \cH$ that result from a given set of values $\{\mhp,\mhm\}_{h \in \cH}$ is referred to as the \emph{uncertainty set} $\cU$. We use $\{\cP_h\}_{h\in \cH}$ to denote the color partition of $\cP$ where $j \in \cP_h$ has color $h$ and we define $\chi^{-1}(h)= \cP_h$. We can analogously define the partition $\{\hat{\cP}_h\}_{h\in \cH}$ for the assignment $\hat{\chi}$ where $\hat{\chi}^{-1}(h)= \hat{\cP}_h.$
More formally, given a valid set of noise parameters $\{\mhp,\mhm\}_{h \in \cH}$ satisfying inequalities \eqref{eq:plus} and \eqref{eq:minus} the uncertainty set $\cU$ is 
\begin{flalign}
    \cU = \{ \hat{\chi} \mid \hat{\chi} \text{ satisfies } \eqref{eq:neg}-\eqref{eq:pos}  \} 
\end{flalign}
\begin{subequations} 
    \begin{equation}\small 
    \label{eq:neg}
   \forall h \in \cH:  |\cP_h \setminus \hat{\cP}_h| \leq  \mhm \\ 
    \end{equation}
    \begin{equation}\small 
    \label{eq:pos}
   \forall h \in \cH:  |\hat{\cP}_h \setminus \cP_h | \leq  \mhp  
    \end{equation}
\end{subequations}

\begin{restatable}{proposition}{uncertaintyset} \label{prop:uncertaintyset}
For any instance with noise parameters $\{ \mhp,\mhm \}_{h \in \cH}$, the group uncertainty set $\cU$ is defined as the set of all group assignments $\hat{\chi}:\cP \rightarrow \cH$ that satisfy the constraints in \eqref{eq:neg} and \eqref{eq:pos}.
\end{restatable}
The above simply states that any $\hat{\chi}$ coloring (element) in the uncertainty set should result in an assignment where (i) the number of points that are assigned a color $h$ by $\chi$ but actually belong to a color $g \neq h$ should not exceed $\mhout$ and (ii) the number of points that are assigned a color $g \neq h$ by $\chi$ but actually have color $h$ is no more than $\mhp.$ 

For the case of two colors (denote them by red and blue), we naturally have $m_{\text{red}}^+ = m_{\text{blue}}^-$ and $m_{\text{red}}^- = m_{\text{blue}}^+$. To see that, note that using inequalities \eqref{eq:plus} and \eqref{eq:minus} it follows that $m_{\text{red}}^+ \leq m_{\text{blue}}^- \leq m_{\text{red}}^+$ and therefore $m_{\text{red}}^+ = m_{\text{blue}}^-$. Similarly, one can show that $m_{\text{red}}^- = m_{\text{blue}}^+$. The new implied equalities are natural as they simply state that what one color loses is gained by the other and vice versa.

To give a sense of the resulting group assignments implied by a given set of values $\mhp$ and $\mhm$ consider the two color toy example shown in Figure \ref{fig:heterogenous_noise}. In this example, we have $m_{\text{red}}^- =2$ whereas $m_{\text{blue}}^- =1$. Further, from the previous discussion since we have two colors then we immediately have  $m_{\text{red}}^+ =1$ and $m_{\text{blue}}^+ =2$. The figure shows the many possible colorings that can result from the given noise values, note that even for this simple example there are a total of $12$ possibilities which is larger than the number of given points $4$. A robust fair clustering has to achieve fairness over all possible colorings (all colorings in the uncertainty set). 
\begin{figure}[H]
    \centering
\includegraphics[scale=0.3]{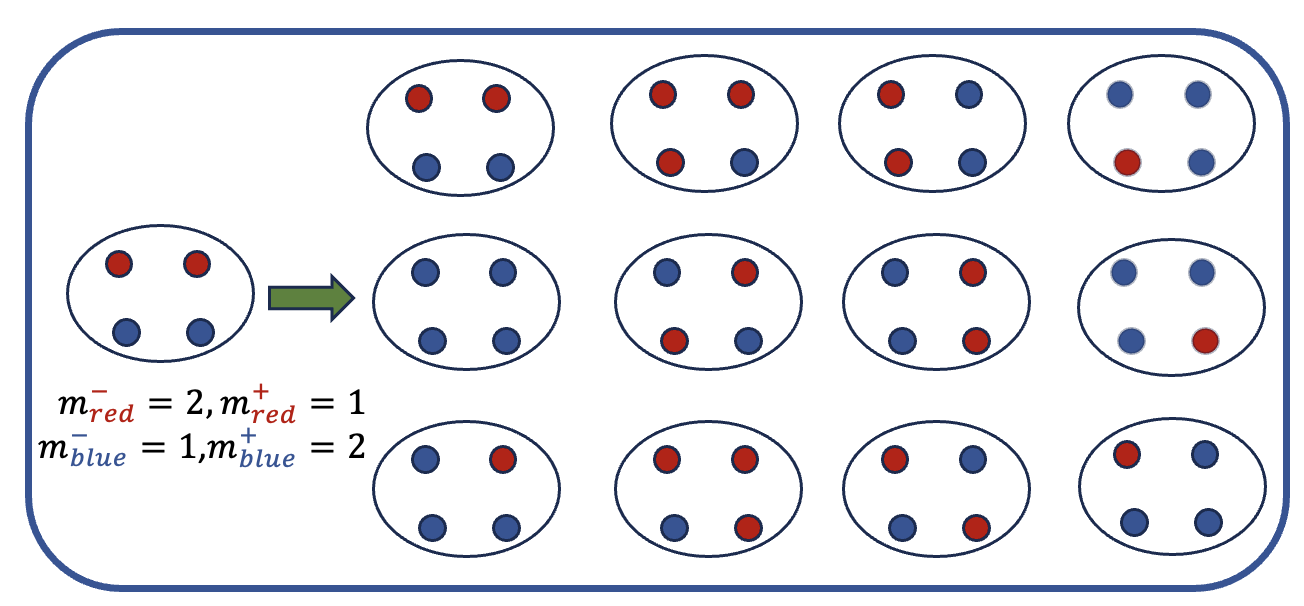}
    \caption{All colorings in the uncertainty set of a toy example with $4$ points and $m_{\text{red}}^-=2, m_{\text{blue}}^-=1$ are shown.}
    \label{fig:heterogenous_noise}
\end{figure}
We further note that a simple possible assignment of the color parameters would set them all to the same value, i.e., $\forall h \in \cH: \mhp=\mhm=m$. This essentially states that any color can increase or decrease by $m$ points. Figure \ref{fig:final-model-fig} shows the same previous example where all noise parameters have been set to $2$. Clearly, the number of possible colorings (size of the uncertainty set) has increased from $12$ to $16$. 
 \begin{figure}[H]
    \centering
\includegraphics[scale= 0.3]{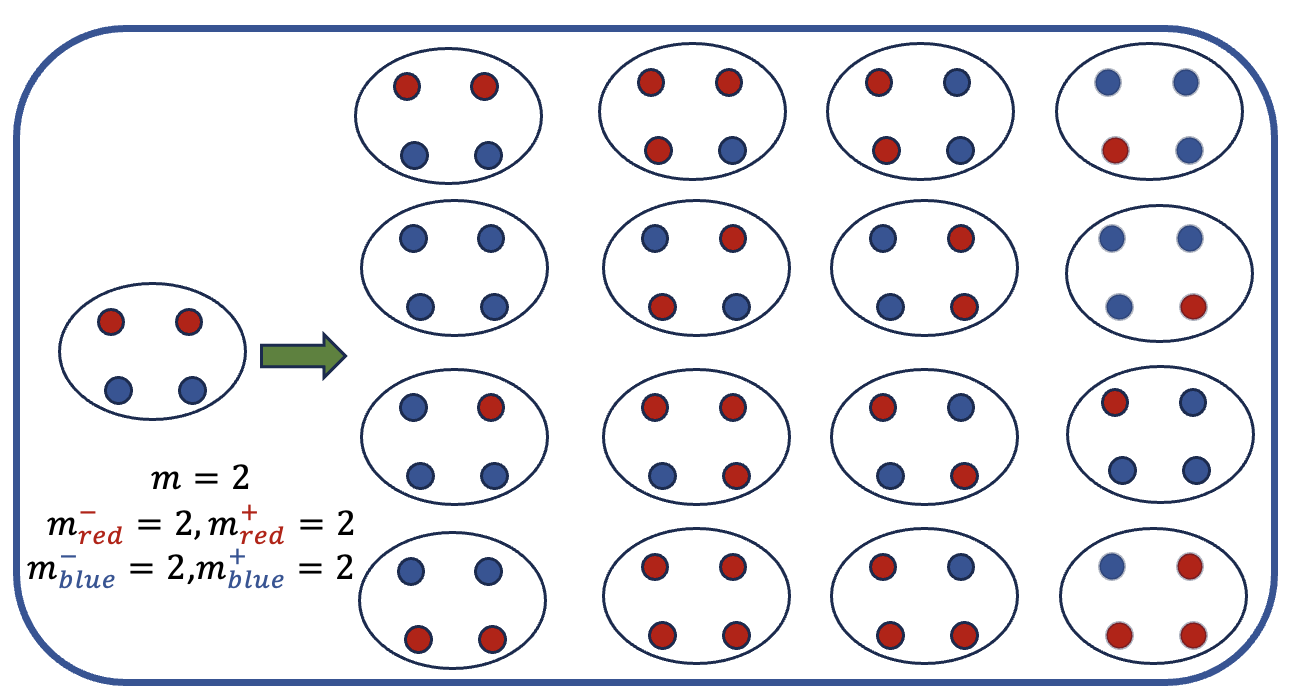}
    \caption{All colorings in the uncertainty set for the same toy example, now with $ m = 2 $. Note the new color assignments in the bottom row where the two (originally) blue points become red.}
    \label{fig:final-model-fig}
\end{figure}
\vspace{-0.5cm}
We are now ready to state our problem. Given an instance of \textit{fair clustering} $(\cP,\chi,k,\cH,\Vec{l},\Vec{u})$ along with noise parameters $\{\mhp,\mhp\}_{h \in \cH}$. The objective of the \rfc{} problem is to find a clustering that minimizes the $k$-center objective while ensuring that the fairness constraints are satisfied for every possible coloring in the uncertainty set $\cU$. Formally, the optimization problem of \rfc{} is 
\begin{subequations} \label{lp:robustfaircluster}   
\begin{equation} \label{eq:cluster_obj_rfc}
\min\limits_{S : |S|\leq k,  \phi}  \ \ \max_{j \in \cP} \ d(j, \phi(j)) 
\end{equation}    
\begin{equation}  \label{eq:robust_fair}
\forall \hat{\chi} \in \cU, \forall i \in S, 
 h \in \cH:   l_h   \leq  \frac{|C_{i,h}(\hat{\chi}) | }{ |C_i|} \leq  u_h 
\end{equation}
\end{subequations}
where $C_{i,h}(\hat{\chi}):= \hat{\chi}^{-1}(h)\cap C_i$ denotes the subset of points in $C_i$ which have been assigned to group $h$ by a coloring $\hat{\chi} \in \cU$. We also define a $\lambda$-violating solution as one where the fairness constraints of \eqref{eq:robust_fair} are violated by at most $\lambda$, formally a  $\lambda$-violating solution satisfies
\begin{equation}  \label{eq:robust_fair_viol}
\forall \hat{\chi} \in \cU, \forall i \in S, 
 h \in \cH:   l_h -\lambda  \leq  \frac{|C_{i,h}(\hat{\chi}) | }{ |C_i|} \leq  u_h +\lambda
\end{equation}
Clearly, smaller $\lambda$ implies a smaller violation of the fairness constraints and any value of $\lambda \ge 1$ is vacuous. 

Finally, we note that our model requires the decision maker to specify a total of at most $2\ell$ many parameters unlike \cite{esmaeili2020probabilistic}  which needs a total of at least $(\ell-1) \cdot n$ many parameters (necessarily growing with the size of the dataset). The values in our model can be simply set using prior knowledge and statistics. Furthermore, if the decision maker does not posses fine-grained knowledge about the noise parameters $\mhp,\mhm$ for each group then deciding one value $m$ and setting $\mhp=\mhm=m$ to upper bound the total change in any group would be sufficient. While this would increase the size of the uncertainty set, a clustering that is robust fair under an uncertainty set remains robust fair under a more restricted one. More formally, given a problem instance and two uncertainty sets $\cU'$ and $\cU$, if $\cU' \subset \cU$ then it follows immediately that a $\lambda$-violating solution under $\cU$ is also a $\lambda$-violating solution under $\cU'$. However, we note that while one may always expand the uncertainty set to ensure fairness for higher noise values it would come at an expense. Specifically, a robust solution would have to ensure fairness over a larger uncertainty set and that would in general lead the optimization objective (which is the clustering cost/quality) to be degraded.

%% file: 04usefulfacts.tex
\section{Algorithm and Theoretical Analysis}  \label{sec:analysis}
We start this section by making a collection of mathematical observations that are essential for our algorithm. First, note that the size of the uncertainty set $|\cU|$ can grow exponentially in the size of the dataset $n$. For example, for the two color case with $\mhin  = \mhout = m$ for both colors, the size of $\cU$ is $\sum_{0\leq i,j\leq m} \binom{n_{1}}{i}\binom{n_2}{j} \geq  \big( \frac{n}{2m} \big)^m$  where $n_1$ and $n_2$ are the number of points for the first and second color.
For a reasonable choice of $m=\alpha n$ where $\alpha$ is a fractional constant in $(0,1)$, it is straight forward to see that the size of $|\cU|=\Omega(c^n)$ where $c$ is a constant strictly greater than $1$. This implies that we can have an exponential set of constraints in \eqref{eq:robust_fair}. Our first critical observation is that we can replace this exponential set by an equivalent polynomially sized set of constraints. 


\begin{restatable}{lemma}{polyconstraints}\label{lemma:poly_constraints}
For any instance of \rfc{} the constraints \eqref{eq:robust_fair} are equivalent to \eqref{eq:ub} and \eqref{eq:lb}.
\begin{subequations} \small
   \begin{equation}
 \forall i\in S, h\in \cH:  \frac{|C_{i,h}({\chi})| + \mhin }{ |C_i|} \leq  u_h  \label{eq:ub} 
 \end{equation} 
 \begin{equation}
        \forall i\in S, h\in \cH:  \frac{|C_{i,h}({\chi})| -\mhout }{ |C_i|}  \geq  l_h \label{eq:lb}
 \end{equation}
\end{subequations}
\end{restatable}
We call the constraints \eqref{eq:ub} and \eqref{eq:lb} the \emph{robust fairness} constraints. To see what the lemma means, consider a color $h$ and its upper proportion bound $u_h$. The lemma essentially states that instead of ensuring that the solution satisfies $\frac{|C_{i,h}(\hat{\chi}) | }{ |C_i|} \leq  u_h$ for all colorings $\hat{\chi} \in \cU$ as done in \eqref{eq:robust_fair}, we may instead take $C_{i,h}({\chi})$ (which uses the given coloring $\chi$) and add the highest (worst-case) increase in the number of points that can be gained by color $h$ which is $\mhp$ and satisfy a single constraint of $\frac{|C_{i,h}({\chi})| + \mhin }{ |C_i|} \leq  u_h$ in \eqref{eq:ub} instead. A similar statement can be made about the lower bound $l_h$ in \eqref{eq:lb}.

As a result of \cref{lemma:poly_constraints}, it follows that a \(\lambda\)-violating solution of constraint \eqref{eq:robust_fair_viol} is only required to satisfy the following reduced constraints.
\begin{subequations} \small
   \begin{equation}
 \forall i\in S, h\in \cH:  \frac{|C_{i,h}({\chi})| + \mhin }{ |C_i|} \leq  u_h +\lambda \label{eq:ub-lamb} 
 \end{equation} 
 \begin{equation}
        \forall i\in S, h\in \cH:  \frac{|C_{i,h}({\chi})|- \mhout }{ |C_i|}  \geq  l_h  -\lambda \label{eq:lb-lamb}
 \end{equation}
\end{subequations}
Another critical observation is that upper and lower proportion bounds that would have a non-empty set of feasible solutions in ordinary fair clustering might lead to an infeasible \rfc{} instance unless the bounds are relaxed by a sufficient margin. To see that, consider an instance of \rfc{} with two red and two blue points and noise parameters $ m_{\text{red}}^+ = m_{\text{blue}}^+ = m = 1 $. If we set the proportion bounds to $ l_{\text{red}} = l_{\text{blue}}=1/2$ and $u_{\text{red}} = u_{\text{blue}}= 1/2 $, then clearly we would have an ordinary (non-robust) fair clustering solution. However, one can see through \cref{lemma:poly_constraints} that no feasible robust fair solution exists. Now, if we relax the bounds to $ l_{\text{red}} = l_{\text{blue}} = 1/3 $ and $ u_{\text{red}} = u_{\text{blue}} = 2/3 $, then a single cluster containing all four points becomes a feasible solution. More formally, the proportions bounds have to be relaxed exactly as shown in the following observation:
\begin{restatable}{observation}{ublbobs}\label{obs:ublb}
    For any instance of the \rfc{} problem, a feasible solution exists if and only if for every group $h\in \cH$, $u_h \ge \frac{\nh + \mhin}{n}$ and $l_h \leq \frac{\nh - \mhout}{n}$.
\end{restatable}

%% file: 05algorithms.tex
\subsection{Our Algorithm: \rfcalg{}} \label{sec:alg_rfc}

To solve our problem we employ a two-stage approach where the initial stage selects the centers and the second stage assigns the points to the centers.  While various prior papers in fair clustering \citet{bera2019fair,bercea2018cost,esmaeili2020probabilistic,esmaeili2021fair} use a similar two-stage approach, the centers used in the first stage are selected by any vanilla (ordinary) $k$-center algorithm. In our case, it is actually critical that the centers are selected carefully by our algorithm (subroutine) \textsc{GetCenters}. In fact, in \cref{subsec:failure_of_vanilla} we show how using another algorithm would break a critical step in our proof. 

Since we are dealing with a $k$-center objective, the optimal radius (the maximum distance from any point to its assigned center) for \rfc{} belongs to a finite set of possible values, i.e., the set of $\binom{n}{2}$ distance values. Therefore, our subroutine \textsc{GetCenters} (Algorithm \ref{alg:filtercenters}) receives as input a ``guessed'' radius value $R$ along with the entire set of points $\cP$. \textsc{GetCenters} outputs a set of centers $S$. The procedure begins with all the points being unmarked, then in each iteration we add an arbitrary unmarked point $j$ to $S$, and mark the all the points at a distance of at most $2R$ from $j$ (this includes point $j$ as well). We repeat this step until all the points are marked.

\vspace{-0.2cm}
\begin{algorithm}[H]
  \caption{\textsc{GetCenters}} \small
\label{alg:filtercenters}
 \textbf{Input}: {Set of points $\cP$, and a radius $R$}\\
 \textbf{Output}: {Cluster centers $S$ }  \\
   \vspace{-0.4cm}
 \begin{algorithmic}[1] 
 \STATE $S \gets \emptyset$
    \WHILE{ $\cP \neq \emptyset$}
    \STATE Pick an arbitrary $j \in \cP$ and $S\gets S\cup \{j\}$
    \STATE $\cP \gets \cP \setminus \text{Ball}(j,2R)$  
    \ENDWHILE
    \RETURN $S$ 
 \end{algorithmic}
 \end{algorithm} 


We will show that the centers returned by \textsc{GetCenters} when run at a sufficiently large value of $R$ satisfy good properties that enable us to post-process them to obtain a \emph{robust fair} clustering. Before that, it is important to introduce the feasibility linear program (LP) which takes a radius value $R$ and a collection of centers $S$ and assigns the points in $\cP$ to centers in $S$. Since in a given fixed instance the set of centers $S$ and radius value $R$ can vary as inputs, we call it $\lp$, its full details are shown below. 
\begin{align} 
  {\lp:} \hspace{3.75cm} &  \nonumber \\ 
    \forall j \in \cP: \sum_{i \in S}x_{i,j} & =1 \label{lp:rfa} \\
    \forall i \in S, h \in \cH: \sum_{j \in \cP^h}x_{i,j} + \mhin &\leq u_h \sum_{j \in \cP}x_{i,j} \label{cons:ub_trf}  \\
    \forall i \in S,  h \in \cH: \sum_{j \in \cP^h}x_{i,j}- \mhout &\geq l_h \sum_{j \in \cP}x_{i,j}  \label{cons:lb_trf} \\
  \forall i \in S,  j\in \cP: \text{ if $d(i,j)>3R$: } x_{i,j} &=0,\text{else:}x_{i,j} \geq 0 \label{eq:endrfa}
\end{align}
For each point $j \in \cP$ and center $i \in S$, $\lp$ has a decision variable $x_{i,j} \in [0,1]$ which denotes the fractional assignment of point $j$ to center $i$ in $S$. $\lp$ is more easily interpreted by considering the integral values of $x_{i,j} \in \{0,1\}$ instead of $[0,1]$. Therefore, constraint \eqref{lp:rfa} simply states that each point should be assigned to exactly one center. Further, it follows that $\sum_{j \in \cP^h}x_{i,j} = |C_{i,h}({\chi})|$ and $\sum_{j \in \cP}x_{i,j}=|C_i|$, hence constraint \eqref{cons:ub_trf} is simply imposing constraint \eqref{eq:ub} of \cref{lemma:poly_constraints} to ensure that the upper proportion bounds are not violated. A similar reasoning follows for constraint \eqref{cons:lb_trf}. The last constraint \eqref{eq:endrfa} simply forbids assigning points $j$ to centers $i$ that are at a distance greater than $3R$, this is done by setting the assignment variables $x_{i,j}=0$ if the distance $d(i,j) > 3R$ and otherwise allowing it to be in $[0,1]$. Note that the $\lp$ receives $R$ as an input parameter but uses $3R$ in constraint \eqref{eq:endrfa}.  

$\lp$ is elaborate and in fact it is not difficult to see that it might not be feasible for an arbitrary set of centers $S$ and an arbitrary radius value $R$. Interestingly, we show that if \textsc{GetCenters} is run at a value of $R \ge R^*$ where $R^*$ is the optimal radius (clustering cost) value then the set of centers $\sgood$ returned by \textsc{GetCenters} satisfies this LP at radius $R$, i.e., $\text{LP}(\sgood,R)$ is \emph{feasible}. This is is shown in the following lemma. In fact, the lemma additionally shows that the number of centers in $\sgood$ is at most $k$, i.e., guaranteeing that we would not have more than $k$ centers. The main idea in the proof is to show that an optimal robust fair solution $(S^*,\phi^*)$ of cost $R^*$ can instead use the centers $\sgood$ at the expense of degrading the clustering cost to $3R^*$. This is done by moving clusters in $(S^*,\phi^*)$ to carefully chosen centers in $\sgood$.
Note that the proof is non-constructive as it assumes knowledge of the optimal solution.

\begin{restatable}{lemma}{silesskone}\label{claim:silessk}
Let the optimal clustering cost (radius) be $R^*$, then if we set $R \geq R^*$ then \textsc{GetCenters} (\cref{alg:filtercenters}) returns a set $\sgood$ such that
    (1) $|\sgood| \!\leq\! k$ and 
    (2) $\text{LP}(\sgood,R)$ is feasible.  
\end{restatable}
The above suggests that we may run \textsc{GetCenters} at different radius values $R$, by Lemma \ref{claim:silessk} once $R \ge R^*$ the returned centers $\sgood$ would have at most $k$ centers and since $\text{LP}(\sgood,R)$ would be feasible by running it we would obtain a feasible assignment. In fact, our algorithm \rfcalg{} (\cref{alg:lbfaircluster}) does that and uses binary search over the set of pairwise distances to find the smallest value of $R$ where the conditions of Lemma \ref{claim:silessk} are satisfied. The issue is that the feasible solution that we would obtain $\mathrm{x}^{\text{LP}}$ can be fractional, i.e., $\mathrm{x}^{\text{LP}}_{i,j} \in [0,1]$ and not necessarily $\in \{0,1\}$. Therefore, we would have to round these fractional values into valid integral ones $\mathrm{x}^{\text{Integ}}_{i,j} \in \{0,1\}$. The following lemma shows that using the \textsc{MaxFlow} rounding scheme \footnote{In short, in \textsc{MaxFlow} rounding we solve a network flow instance corresponding to a given clustering instance and fractional LP assignment $\mathrm{x}^{\text{LP}}$. Then an integral flow is found and used to construct the rounded integral assignment $\mathrm{x}^{\text{Integ}}$.}  \cite{bercea2018cost,dickerson2023doubly} (see \cref{sec:maxflow} for more details) we can obtain an integral assignment at no increase to the clustering cost and only for a slight change in the cluster sizes as shown in \cref{lemma:obj_rounded}.

\begin{lemma} \label{lemma:obj_rounded} 
Let $\mathrm{x}^{\text{Integ}}$ be the integral assignment that results from running \textsc{MaxFlow} rounding over a fractional assignment $\mathrm{x}^{\text{LP}}$, then  (1) if $\mathrm{x}^{\text{LP}}_{i,j}=0$ then $\mathrm{x}^{\text{Integ}}_{i,j}=0$ and  (2) for any center $i\in \hat{S}$ and group $h \in \cH$, 
\begin{align*}
      \floor{|C_i^{\text{LP}}|} \leq & |C_{i}^{\text{Integ}} | \leq  \ceil{|C_i^{\text{LP}}|}, \\ 
      \floor{|C_{i,h}^{\text{LP}}|} \leq &|C_{i,h}^{\text{Integ}}|\leq \ceil{|C_{i,h}^{\text{LP}}|}  
\end{align*}
where $|C_i^{\text{LP}}|= \sum_{j\in \cP}\mathrm{x}_{i,j}^{\text{LP}}$, $|C_i^{\text{Integ}}|= \sum_{j\in \cP}\mathrm{x}_{i,j}^{\text{Integ}}$,  $|C_{i,h}^{\text{LP}}|=\sum_{j\in \cP_h}\mathrm{x}_{i,j}^{\text{LP}}$, and  $|C_{i,h}^{\text{Integ}}|=\sum_{j\in \cP_h}\mathrm{x}_{i,j}^{\text{Integ}}$
\end{lemma}
The fact that the clustering cost would not increase should be clear from guarantee (1) of the above lemma as it implies that $\mathrm{x}^{\text{Integ}}$ will only assign points $j$ to centers $i$ where they already had a non-zero assignment in the fractional solution, i.e., $\mathrm{x}^{\text{LP}}_{i,j}>0$. The integral assignment $\mathrm{x}^{\text{Integ}}$ can immediately be used to construct the assignment function $\hat{\phi}: \cP \to \hat{S}$. Therefore, our final solution is $(\hat{S},\hat{\phi})$.  

All that remains is the final guarantee on the solution $(\hat{S},\hat{\phi})$. While it is clear that the radius is at most $3R^*$, the theorem below also shows that the violation in the fairness constraints is also bounded by a small value. Formally, we have  the following theorem.

\begin{restatable}{theorem}{finaltheorem}
\label{thm:final}
\rfcalg{} (\cref{alg:lbfaircluster}) is a $3$-approximation algorithm for the \rfc{} with fairness violation $\lambda= \frac{2}{\outm}$. 
\end{restatable}
From the above theorem it is clear that for large values of $\outm$ the violations would become smaller. In fact, if $\outm = \omega(1)$ then  $\lambda \to 0$ as $n \to \infty$.


 \begin{algorithm}[H]
  \caption{\rfcalg{}}
\label{alg:lbfaircluster}
 \textbf{Input}: {An instance of \fc{}, ${\mhin, \mhout}_{h \in \cH}$. }\\
\textbf{Output}: {Clustering of points $(\hat{S},\hat{\phi})$} \\
  \vspace{-0.5cm}
 \begin{algorithmic}[1] 
 \STATE Perform binary search to find the smallest radius $R$ for which the set $S$ returned by \textsc{GetCenters}($\cP, R$) has at most $k$ centers and $\lp$ is feasible.   
 \STATE We call the set of centers returned at the smallest radius $\hat{S}$ and we call its associated fractional assignment $\mathrm{x}^{\text{LP}}$. 
 \STATE Round $\mathrm{x}^{\text{LP}}$ to $\mathrm{x}^{\text{Integ}}$ using \textsc{MaxFlow} rounding (see full details in \cref{sec:maxflow}). 
 \item Construct the assignment $\hat{\phi}: \cP \to \hat{S}$ as follows, for each $j \in \cP$ set $\hat{\phi}(j) = i$ if $\mathrm{x}^{\text{Integ}}_{i,j}=1$. 
 \RETURN $(\hat{S},\hat{\phi})$
 \end{algorithmic}
 \end{algorithm}

\subsection{Failure When Using a Vanilla Clustering Algorithm}  \label{subsec:failure_of_vanilla}
Here we show that $\lp$ would not be feasible using centers selected by a vanilla clustering algorithm $\vS$\footnote{By a vanilla clustering algorithm we mean one which has some $\alpha$ approximation ratio for the ordinary clustering objective.}. This is shown in the theorem below. The main idea behind this is that a vanilla clustering algorithm lacks adaptivity and therefore may select too many centers and since the constraints in $\lp$ (specifically \eqref{cons:ub_trf} and \eqref{cons:lb_trf}) implicitly impose a lower bound on the cluster size there would not be enough points to assign to each center. While closing a subset of centers in $\vS$ might lead to a feasible $\text{LP}(\vS,R)$, knowing which ones to close without degrading the clustering cost is not straightforward to do. Our algorithm \textsc{GetCenters} avoids all of this and gives a simple way to find the set of centers and construct the final clustering solution. 
\begin{restatable}{theorem}{vanillabad}\label{theorem:vanilla_not_good}
Given a set of centers selected by a vanilla $k$-center algorithm $\vS$ then for any arbitrarily large $R$, there may not exist a feasible solution to $\text{LP}(\vS,R)$.
\end{restatable}

%% file: 06experiments.tex
\section{Experiments}\label{sec:experiments}
\begin{figure*}[h!]
\centering
  \includegraphics[width=1.0\textwidth]{Figures/experiments/euclidean_wide_large_m_m_vs_fairvio.png}
\caption{Plots for $k$-center objective and fairness violation as $m/n$ increases. Over half of \rfcalg{}'s pictured fairness violations are exactly zero; the rest fall between $10^{-5}$ and $10^{-3}$. The 95\% confidence interval around \probalg{} is shaded; however, it is faint because the fairness violations are very sharply concentrated around their plotted mean.}
    \label{fig:m-and-T-plots}
\end{figure*}
We conduct experiments on a commodity laptop (Ryzen 7 5800, 16GB RAM) using Python 3.6 and cplex 12.8 to solve LPs.
Additional details and plots are available in Appendix~\ref{app:experiments}.

\textbf{Datasets.} We experiment on three datasets from the UCI repository \citep{dua2017uci}: \adult, \bank, and \cens. The datasets have 32k, 32k and 4.5k points with 5, 3, and 66 numerical features, respectively. Further, the colors in each dataset, i.e., sensitive attributes, are respectively a binary sex, a binary marital status, and a membership in one of three age buckets. The distances between any pair of points is set to the Euclidean distance between their normalized numerical features.

\textbf{Parameters.} The experiments use noise parameters $\mhin = \mhout = m, \forall h \in \cH$, where $m$ ranges from $\frac{n}{100}$ to $\frac{n}{10}$. We set the proportions $l_h$ and $u_h$ to the respectively greatest and least values leading to feasible \rfc{} instances, in accordance with Observation~\ref{obs:ublb}. 
Therefore, $l_h$ and $u_h$ are the same across instances in the same dataset. The number of centers is fixed at $10$, i.e., $k=10$. 

We benchmark our proposed algorithm \rfcalg{} against two relevant baselines. Namely, \emph{probabilistic} fair clustering \citep{esmaeili2020probabilistic} and \emph{deterministic} fair clustering \citep{bera2019fair}. These three fair clustering algorithms are denoted as \rfcalg{}, \probalg{}, and \detalg{} in the plots. Note that probabilistic fair clustering is an algorithm for two colors only so it is tested solely on \adult and \bank.

We evaluate the algorithms on their attained $k$-center objectives and fairness violation given by \eqref{eq:robust_fair_viol}.
For deterministic (and robust) fair clustering we obtain fairness violations by directly finding the corruption of point colors leading to the greatest fairness violation as described in Inequalities~\eqref{eq:ub-lamb} and \eqref{eq:lb-lamb}. Specifically, for each color, up to $m$ points  can essentially be added or subtracted.

In the probabilistic fair clustering work of \citet{esmaeili2020probabilistic}, the noise model is different. Specifically, each point's color is corrupted with some probability. To ensure that our evaluation is fair we evaluate the probabilistic instances under their assumed noise model. Instances are set up so that each point's color is corrupted with probability $m/n$, leading to an \textit{expected} $m$ corruptions dataset-wide. In contrast, robust and deterministic fair clustering allow for up to $m$ corruptions in \emph{each color}.
Finally, we sample 200 realizations of colors and directly report the mean fairness violations. However, probabilistic fair clustering says nothing about point correlations. We exploit this fact as follows. First, sample $S_{i,h} \sim \operatorname{Bernoulli}(m/n)$ for each cluster $i$ and color $h$, and then, if $S_{i,h}=1$, corrupt the colors of all points of color $h$ assigned cluster $i$.
Note that this is a generous evaluation, since an adversarial color assignment (as done in deterministic and robust fair clustering) can only lead to a higher violation. 

 \begin{figure}[h!]
 \centering
   \begin{minipage}{0.52\textwidth}   
        \centering \includegraphics[width=1.0\textwidth]{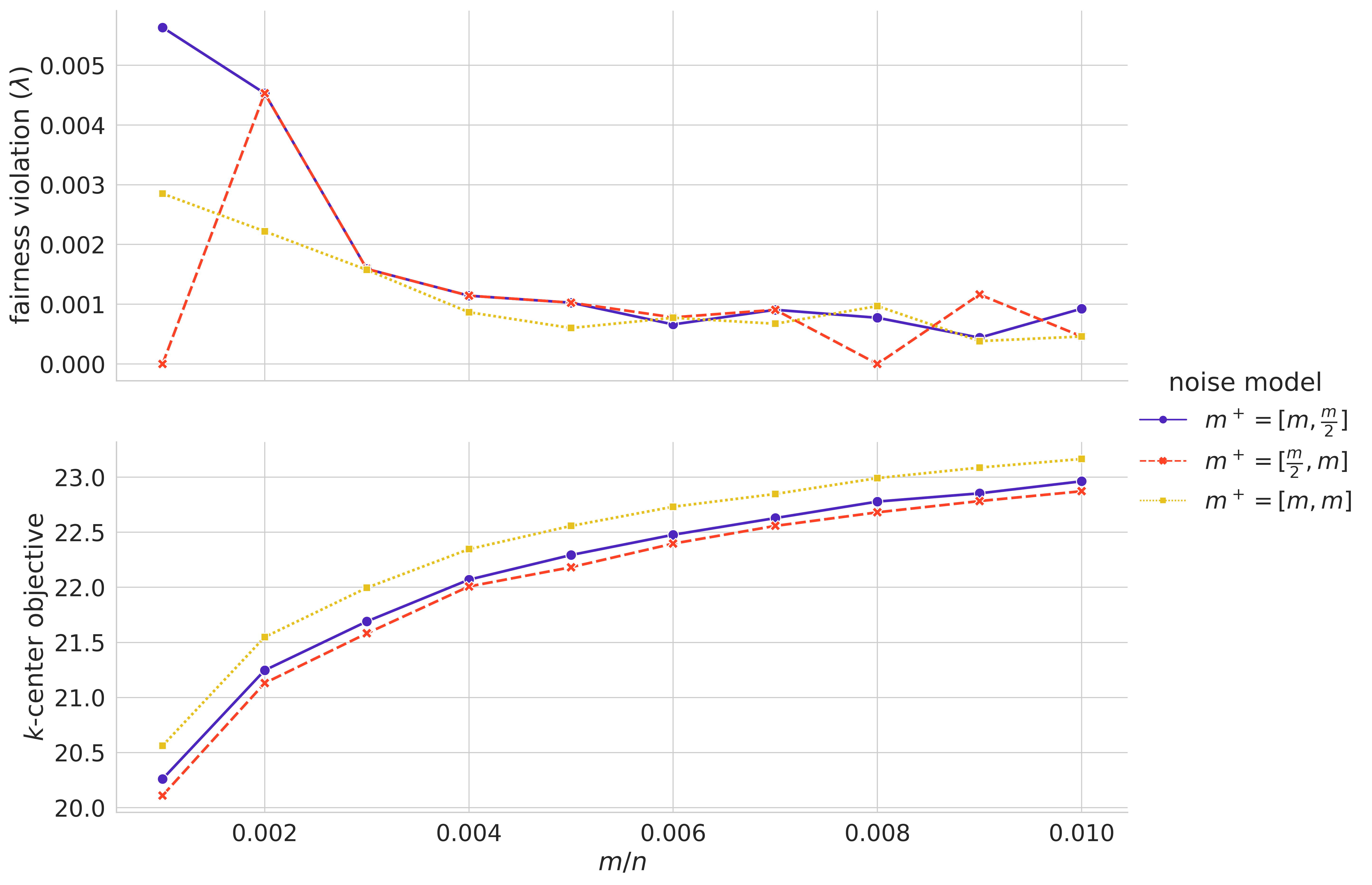}
        \caption{Comparison of the effect of increasing noise, given different noise model configurations.}
        \label{fig:error-models-plot}
    \end{minipage}
 \end{figure}
 
Figure~\ref{fig:m-and-T-plots} shows the results of our experiments.
\rfcalg{} \emph{maintains fairness violations of zero and near-zero} across the board; unlike deterministic and probabilistic baselines which have fairness violations as large as $0.6=60\%$.
In fact, in \detalg{} the post-corruption ratio of some colors becomes 0 in \adult  and in \bank (i.e., a color is completely absent from the cluster) or 1 in \cens (i.e., a color is fully dominating the cluster). Note that these are the worst-possible fairness violations in all three datasets. \probalg{} nearly hits these worst-case violations as well.


The objective (clustering cost) of \rfcalg{} is greater and increases with $m/n$, as expected when targeting a more stringent notion of robustness.
The break in the objective plot of Figure~\ref{fig:m-and-T-plots} occurs when \rfcalg{} opens fewer centers.
As $m$ increases, centers must have a greater number of points and thus fewer centers can receive points; otherwise applying all $m$ corruptions to the smallest center results in (nearly) absent colors or (nearly) monochromatic centers, which produce large fairness violations.

Finally, Figure~\ref{fig:error-models-plot} concludes with experiments on \bank using three settings of noise parameters. 
We fix the $l_h$ and $u_h$ for all plots as described before. We denote the two colors in bank by $0$ and $1$. Further, we take $m$ from $\frac{1}{1000}n$ to $\frac{1}{100}n$ but consider the $(m^+_0, m^+_1)$ choices of $(m, m/2)$, $(m/2, m)$, and $(m, m)$.
The experiments validate our intuitions.
First, $(m,m)$ achieves the highest objective as it has the most corruptions (largest uncertainty set).
Moreover, $(m/2, m)$ and $(m, m/2)$ both have the same number of total corruptions and their objectives are indeed comparable.
In all cases, the fairness violations border on zero, agreeing with Theorem~\ref{thm:final}.


%% file: 07appendix.tex
\appendix
\onecolumn

\section{Useful Fact}  
\begin{fact} \label{cauch_ineq}
    For any positive real numbers $a_1,a_2,\dots, a_n$ and $b_1,b_2,\dots,b_n$, the following holds
    \begin{equation}
        \min_{i  \in [n]} \frac{a_i}{b_i} \leq \frac{a_1+a_2 + \cdots + a_n}{b_1+b_2 + \cdots + b_n } \leq \max_{i  \in [n]} \frac{a_i}{b_i}
    \end{equation}
\end{fact}
\begin{proof}
Let $\tau_{\max} = \max_{i  \in [n]} \frac{a_i}{b_i}$, therefore we have 
\begin{align*}
    \frac{a_1+a_2 + \cdots + a_n}{b_1+b_2 + \cdots + b_n } \leq \frac{\tau_{\max} (b_1+b_2 + \cdots + b_n)}{b_1+b_2 + \cdots + b_n} = \tau_{\max} 
\end{align*}
The lower bound can be proved similarly. 
\end{proof}

\section{Omitted Proofs}
\label{app:prelims}

\uncertaintyset*

\begin{proof}
Let $\hat{\chi}$ be a feasible assignment in the uncertainty set and let $\hat{\cP}_h = \hat{\chi}^{-1}(h)$ be the collection of points assigned color $h$ by $\hat{\chi}$. For any color $h$, $\cP_h \setminus \hat{\cP_h}$ denotes the set of points that are assigned color $h$ by $\chi$ and a different color $g \neq h$ by $\hat{\chi}$. Clearly, by definition of $\mhm$ the first constraint should hold, i.e. 
\begin{align}
    |\cP_h \setminus \hat{\cP_h}| \leq \mhm
\end{align}
Furthermore, $\hat{\cP_h} \setminus \cP_h$ denotes the set of points that are assigned color $h$ by $\hat{\chi}$ but were assigned a different color $g \neq h$ by $\chi$. By definition of $\mhp$ the second constraint holds, i.e. 
\begin{align}
    |\hat{\cP_h} \setminus \cP_h| \leq \mhm
\end{align}

\end{proof}

\begin{restatable}{observation}{clusteri}  \label{obs:ci_lb}
    Suppose that $(S,\phi)$ is $\trf{}$ solution to a given input instance and $u_h<1$ and $l_h>0$. Then the clustering $\{C_1,C_2,\cdots,C_{k'}\}$ of points induced by $(S,\phi)$ must satisfy the following,
    \begin{itemize}
        \item For any $i\in S$, $h \in \cH$, 
        $|C_{i,h}(\chi)| > \mhout$. \label{cih_lowerbound}
        \item For any $i\in S$, $h \in \cH$, 
        $\sum_{g\neq h, g\in \cH} |C_{i,g}(\chi)| >  \mhin$. \label{cih_ubound}
    \end{itemize}
\end{restatable}
\begin{proof}
    We know that since the clustering is robust fair it must satisfy the constraints in \cref{eq:robust_fair} i.e., 
    \begin{align*}
        \forall  i \in S, \forall h \in \cH : \min \frac{|C_{i,h}(\hat{\chi})|}{|C_i|} \geq l_h \implies \min {|C_{i,h}(\hat{\chi})|} >0
    \end{align*}
However, we know that there can be as much as $m_h^-$ points that have incorrect group memberships from each color $h$. Therefore we have
\begin{align*}
       \forall  i \in S, \forall h \in \cH :\min {|C_{i,h}(\hat{\chi})|} >0 \implies  {|C_{i,h}({\chi})|} - m_h^- >0 
\end{align*}
Thus, we show the first part of the claim. 

For each $i \in S$ and color $h \in \cH$, using the lower bound on ${|C_{i,g}({\chi})|}$ obtained in the first part, we can derive a lower bound on $\sum_{g \in \cH, g\neq h} {|C_{i,g}({\chi})|}$ by summing over all the groups $g\in \cH$ and $g\neq h$ as follows,
\begin{align*}
 \sum_{g \in \cH, g\neq h} {|C_{i,g}({\chi})|}  > \sum_{g \in \cH, g\neq h} m_g^-  \geq \mhin 
\end{align*}
The last inequality is from the inequality \eqref{eq:plus} which says that the number of points gained by any group $h$ is at most the total number of points lost by the remaining groups $g \neq h$ i.e., $\sum_{g \in \cH, g\neq h} m_g^-  \geq \mhin$. Therefore, we get the desired bound.
 \end{proof}

\polyconstraints*
\begin{proof}
To show that the constraints \eqref{eq:robust_fair} are equivalent to \eqref{eq:ub} and \eqref{eq:lb}, we have to show that for any feasible clustering $\{C_1,C_2,\dots,C_{k'}\}$ with $k' \leq k$ satisfying \eqref{eq:robust_fair} must also satisfy \eqref{eq:ub} and \eqref{eq:lb} and vice versa. 

First we show the forward direction, i.e., if $\{C_1,C_2,\dots,C_{k'}\}$ satisfies \eqref{eq:robust_fair} then it also satisfies \eqref{eq:ub} and \eqref{eq:lb}. Since the fairness constraints in \eqref{eq:robust_fair} must hold for all points in the uncertainty set $\cU$, they are also valid for the worst-case $\hat{\chi}$ in $\cU$ i.e., 
\begin{align*}
    \forall \hat{\chi} \in \cU, \quad \frac{|\cih(\hat{\chi})|}{|C_i|} \leq u_h 
&\implies \max_{\hat{\chi} \in \cU } \frac{|\cih(\hat{\chi})|}{|C_i|} \leq u_h.
\end{align*}
However, in each cluster $C_i$, we know that the maximum number of points from color $h$ that are incorrectly labeled as $h$ but actually belong to a different group $g$ given by $\min\{ \mhin, \sum_{{g \in \cH , g \neq h}} |C_{i,g}(\hat{\chi})|\}.$ Therefore, for each $ 1\leq i \leq k'$ and $ h \in \cH $, we have
\begin{align*}
    \max_{\hat{\chi} \in \cU } \frac{|\cih(\hat{\chi})|}{|C_i|} =   \frac{|\cih(\hat{\chi})| + \min \{\mhin,\sum_{{g \in \cH , g \neq h}} |C_{i,g}(\hat{\chi})| \} }{|C_i|} = 
   \frac{|\cih(\hat{\chi})| + \mhin}{|C_i|} \leq u_h.
\end{align*}
The last equality is from the fact in \cref{obs:ci_lb} that $\sum_{{g \in \cH , g \neq h}} |C_{i,g}(\hat{\chi})| > \mhin $ 

Next we prove that if a clustering satisfies constraints in \eqref{eq:robust_fair} then it also satisfies \eqref{eq:lb}. Since the clustering is feasible we know that, 
\begin{align*}
    \forall \hat{\chi} \in \cU, \quad \frac{|\cih(\hat{\chi})|}{|C_i|} \geq l_h 
&\implies \min_{\hat{\chi} \in \cU } \frac{|\cih(\hat{\chi})|}{|C_i|} \geq l_h.
\end{align*}

We know that in each $C_i$, the maximum number of points mistakenly assigned a color $h$ but actually belonging to group $ g \neq h$ is $ \min\{\mhout, |C_{i,h}(\chi)|\}$. However from the first part of \cref{obs:ci_lb} it is clear that, for any cluster $C_i$, we have $|C_{i,h}(\chi)| > \mhout$ which implies that $ \min\{\mhout, |C_{i,h}(\chi)|\} = \mhout.$ Therefore, we have
\begin{align*}
    \min_{\hat{\chi} \in \cU } \frac{|\cih(\hat{\chi})|}{|C_i|} = \frac{|C_{i,h}({\chi})| - \min \{\mhout, |C_{i,h}(\chi)|\}}{|C_i|}  = \frac{|\cih({\chi})| - \mhout }{|C_i|} \geq l_h. 
\end{align*}

This concludes the forward direction of our proof. Next, we show that if \eqref{eq:lb} and \eqref{eq:ub} holds, then \cref{eq:robust_fair} also holds. 
We know that for any $1\leq i\leq k'$, $C_i$ satisfies the upper bound constraints in \eqref{eq:ub}. This implies that any feasible assignment $\hat{\chi}$ contains at most $|\cih({\chi})| + \mhin$ points of color $h$ in $C_i$. From \eqref{eq:pos} in \cref{prop:uncertaintyset} we know that, for any assignment in the uncertainty set at most $\mhin$ additional points are assigned to color $h$. Therefore, we have
\begin{align*} 
    \forall h \in \cH : \frac{|\cih({\chi})| + \mhin }{|C_i|} \leq u_h \implies \forall \hat{\chi} \in \cU: \frac{|\cih({\hat{\chi}})|  }{|C_i|}  \leq  \frac{|\cih({\chi})| + \mhin }{|C_i|} \leq u_h.
\end{align*}
For the case of lower bound constraints we know that $C_i$ satisfies the constraints in \eqref{eq:lb}. Therefore, we can say that any feasible  assignment cannot have less than $|\cih(\chi)| -\mhout $ points of color $h$ in $C_i$. Therefore, for any $1\leq i \leq k'$, 
\begin{align*}
    \forall h \in \cH : \frac{|\cih({\chi})| - \mhout }{|C_i|} \geq l_h \implies  \forall \hat{\chi} \in \cU, \quad \frac{|\cih(\hat{\chi})|}{|C_i|} \geq   \frac{|\cih({\chi})| - \mhout }{|C_i|}  \geq     l_h
\end{align*}
as desired. Therefore, the constraints in \eqref{eq:robust_fair} are equivalent to \eqref{eq:ub} and \eqref{eq:lb}.
\end{proof}


\ublbobs*
\begin{proof}
Suppose we have a clustering $\{C_1,\dots,C_{k'}\}$ ) with $k' \leq k$. As before $\cih(\chi)$ denotes the set of points belonging to group $\cH$ in the original coloring.  
Applying the Fact \ref{cauch_ineq}, we get the following,
\begin{equation}\label{eq:color_proportion_bounds}
    \min\limits_{i \in [k']} \frac{|\cih(\chi)|}{|C_i|} \leq\frac{ \sum_{i \in [k']}\cih(\chi)}{\sum_{i \in [k']} C_i} = \frac{|\cP_h|}{|\cP|} \leq \max\limits_{i \in [k']} \frac{|\cih(\chi)|}{|C_i|}
\end{equation}
Suppose that $u_h < \frac{\nh+ \mhin }{n}$ for some color, then by the definition of the robust optimization problem, it is possible to have $\nh+ \mhin $ many points of color $h$ in the dataset. Accordingly, we it must be that for some color assignment $\frac{|\cP_h|}{|\cP|}=\frac{\nh+ \mhin }{n}$ but \eqref{eq:color_proportion_bounds} implies that there exists some value $i \in [k']$ such that $\frac{|C_{i,h}(\chi)|}{|C_i|} \ge  \frac{|\cP_h|}{|\cP|}=\frac{\nh+ \mhin}{n}$. Therefore, $|C_{i,h}(\chi)|> u_h |C_i|$ and therefore the solution is infeasible. The same argument can be made for the lower bound as well. 

We will now show that if $\forall h\in \cH: u_h \ge \frac{\nh+ \mhin }{n}, l_h \leq \frac{\nh-\mhout}{n}$, then the problem must be feasible. To show feasibility we simply show one feasible solution. Specifically, a solution which is always feasible is a one cluster solution that includes all of the points, i.e. $\{C_1\}=\{\cP\}$. Clearly, we have for any color $h: \frac{\nh-\mhout}{n} \leq \frac{|\cP_h|}{|\cP|} \leq \frac{\nh+ \mhin}{n}$. Since the $u_h \ge \frac{\nh+\mhin }{n}, l_h \leq \frac{\nh- \mhout}{n}$, then the solution is feasible. 
\end{proof}

\silesskone*
\begin{proof}
    We first show that for any $R\geq R^*$ the number of centers in $\hat{S}$ returned by \cref{alg:filtercenters} is at most $k$.
    Each center $i^* \in {S^*}$ has at most one $i' \in \hat{S}$ from its optimal cluster. This is because any two centers in $\hat{S}$ are separated by a distance strictly greater than $2R$ and therefore no two centers in $\hat{S}$ are selected from the same optimal cluster. Therefore, we have $|\hat{S}|\leq |S^*| \leq k$. 
    
    To prove the second part of the lemma, it suffices to show that for any $R \geq R^*$, (i) there exists an assignment $\phi': \cP \rightarrow \hat{S}$ that assigns points in $\cP$ to the centers in $\hat{S}$ and (ii) the cluster $\phi'^{-1}(i)$ corresponding to each $i \in \hat{S}$ is \trf{}. In \cref{lemma:existence_lemma}, we show using a non-constructive proof that for any $R \geq R^*$ there exists a feasible solution $(\hat{S},\phi')$ to our problem with a cost of at most $3R$ where each of the centers $i \in \hat{S}$ has a cluster that is \trf{}. 


    
    If such a solution $(\hat{S},\phi')$ exists then it immediately corresponds to feasible solution to the LP, this can be shown as follows: for each point $j \in \cP$, set $x_{i,j}=1$ if $\phi’(j)=i$ and $x_{i,j}=0$ otherwise. Clearly, this $x$ satisfies the constraints in \eqref{lp:rfa}. Moreover, each cluster $C_i = \phi'^{-1}(i)$  satisfies the constraints in \eqref{eq:lb} and \eqref{eq:ub} since it is robust fair, i.e.,
\begin{equation}
    \forall i \in \hat{S}, h \in \cH:  |C_i \cap \cP_h| - \mhout \geq l_h |C_i|, \quad \text{and} \quad |C_i \cap \cP_h| + \mhin \leq u_h |C_i|. \label{eq:feasiblesol}
\end{equation}
Furthermore, since $|C_i| = \sum_{j \in \cP}x_{i,j}$ and $|C_i \cap  \cP_h| = \sum_{j \in \cP_h}x_{i,j}$, the assignment $x$ satisfies the constraints in \eqref{cons:ub_trf} and \eqref{cons:lb_trf} by substituting $|C_i|=\sum_{j \in \cP}x_{i,j}$ and $|C_i \cap \cP_h|=\sum_{j \in \cP_h}x_{i,j}$ in the \eqref{eq:feasiblesol} we have
\begin{align*}
        \forall i \in \hat{S}, h \in \cH:  \sum_{j \in \cP_h}x_{i,j} - \mhout \geq l_h \sum_{j\in \cP}x_{i,j} \quad \text{and} \quad  \sum_{j \in \cP_h}x_{i,j}  + \mhin \leq u_h \sum_{j\in \cP}x_{i,j}. 
\end{align*}
Further note that by \cref{lemma:existence_lemma} that each point is assigned to a center in $\sgood$ that is at most at a distance of $3R$ therefore \eqref{eq:endrfa} is satisfied as well. Therefore, we conclude that for any $R \geq R^*$, there exists a non-empty feasible solution to LP~$(\hat{S},R)$.
\end{proof}

\begin{restatable}{lemma}{existencelemma} \label{lemma:existence_lemma}
For any $R\! \geq R^*$ and set of centers $\hat{S}$ returned by the \textsc{GetCenters} subroutine, (i) there exists an assignment $ \phi': \cP \!\! \rightarrow \!\! \hat{S}$ at a clustering cost of at most $3R$, i.e., $\text{cost}(\hat{S},\phi') \leq 3R$ and (ii) the assignment leads each center $i \in \hat{S}$ to have a cluster that is \trf{}.

\end{restatable}    
\begin{proof}
Suppose that we have an optimal solution $(S^* ,\phi^*)$ with cost $R^*.$ Each point $j$ is assigned to some center $i^* \in S^*$ by the optimal assignment $\phis$. Let $C_{i^*}$ denote the set of points assigned to $i^*$, i.e., $C_{i^*}={\phi^*}^{-1}(i^*)$. 
We show the existence of an assignment $\phi' : \cP \rightarrow \hat{S}$  from the set of points $\cP$ to the set of centers $\hat{S}$ returned by \cref{alg:filtercenters} such that
 (i) the assignment $\phi'$ has $\text{cost}(\sgood,\phi') \leq 3 R$ and  
 (ii) each cluster corresponding to a center $i \in \hat{S}$ is \trf{}. 
 Note that the proof is non-constructive since it assumes knowledge of the optimal solution. 

 We construct the assignment $\phi'$ as follows: $\phi'$ assigns all the points in each cluster $C_{i^*}$ to the center $i' \in \hat{S}$ if $i'$ belongs to the cluster $C_{i^*}$, i.e., if $\phi^*(i') = i^*.$  It is possible that there are clusters $C_{i^*}$ with no point $i' \in \hat{S}$. Therefore, we assign such clusters to a center $i'$ in $\hat{S}$ where $d(i^*,i') \leq 2R$, i.e., $i'$ is at a distance of at most $2R$ from the cluster's center. Since every point has at least one center in $\hat{S}$ at a distance of at most $2R.$ This concludes the description of our assignment $\phi'.$ 
 
 Notice that each $i' \in \hat{S}$ gets assigned all the points associated with at least one center $i^* \in S^*$. This is because any two centers in $\hat{S}$ are separated by a distance strictly greater than $2R$ and therefore no two centers in $\hat{S}$ are selected from the same optimal cluster. Further, each center $i^*\in S^*$ gets assigned to some $i' \in \hat{S}$ at a distance of at most $2R.$ We now prove that this new assignment has a cost of at most $3R$, i.e., $\text{cost}(\hat{S},\phi') \leq 3R$. This holds because for any point $j \in \cP$ we have: 
\begin{align} 
    d(j, \phi'(j)) 
    &\leq d(j, \phis(j)) + d(\phis(j), \phi'(j)) \label{eq:1} \\  
    &\leq R^* + d(\phis(j),\phi'(j)) \label{eq:2} \\ 
    &\leq R^* + 2R \label{eq:3} \\
    &\leq 3R.
\end{align}
Inequalities \eqref{eq:1} follows from triangle inequality. Inequality \eqref{eq:2} follows from the fact that $\phi^*$ is an optimal \trf{} assignment. Finally, inequality \eqref{eq:3} is from the fact that $d(\phi^*(j), \phi'(j)) \leq 2R$ since each center $i^*\in S^*$ is assigned to a center $i' \in \hat{S}$ at a distance of at most $2R.$  It remains to show that for each $i'\in \hat{S}$ the corresponding cluster $\phi'^{-1}(i')$ is \emph{robust fair}, i.e., satisfying the constraints \eqref{eq:lb} and \eqref{eq:ub}. 
The following claim concludes the proof of this lemma. 

\begin{claim}\label{claim:robustfair}
The clustering induced by $(\hat{S},\phi')$ leads each center $i' \in \hat{S}$ to be robust fair.    
\end{claim}
\begin{proof}
Recall that  $C_{i^*}$ denotes the set of points that are assigned to center $i^* \in S^*$ in the optimal clustering $(S^*,\phi^*)$. Since $(S^*,\phi^*)$ is an optimal clustering to our problem, it must satisfy the constraints in \eqref{eq:ub} and \eqref{eq:lb}. Therefore, we have the following:
    \begin{align*}
      \forall i^* \in S^*, \forall h \in \cH:
     \frac{|C_{{i^*},h}(\chi)| + \mhin}{|C_{i^*}|} \leq  u_h, \text{ and }     \frac{|C_{{i^*},h}(\chi)|- \mhout}{|C_{i^*}|} \geq l_h.  
    \end{align*}
    For each $i' \in \hat{S}$, let $N(i')$ denote the set of centers $i^*\in S^*$ that are assigned to $i'$ by $\phi'$. For each $i' \in \hat{S}$ we can upper bound the proportion of any color as follows:
    \begin{align*}
          \frac{|C_{i',h}|+\mhin}{|C_{i'}|}  &= \frac{\big(\sum_{ i^* \in N(i') } |C_{i^*,h}|   
          \big) +\mhin }{\sum_{ i^* \in N(i') } |C_{i^*}|} \\  &\leq \frac{\sum_{ i^* \in N(i') } \big( |C_{{i^*},h}| +  \mhin \big) }{\sum_{ i^* \in N(i') } |C_{i^*}|} \\ &\leq \max_{ i^* \in N(i') } \frac{|C_{{i^*},h}| + \mhin}{|C_{i^*}|}  \leq u_h
    \end{align*}
    Similarly, we can also lower bound the proportions:
    \begin{align*}
         \frac{|C_{i',h}|-\mhout}{|C_{i'}|}  &= \frac{\big(\sum_{ i^* \in N(i') } |C_{i^*,h}|\big) -\mhout  }{\sum_{ i^* \in N(i') } |C_{i^*}|}  \\ &\geq  \frac{\sum_{ i^* \in N(i') } \big( |C_{i^*,h}| -\mhout  \big) }{\sum_{ i^* \in N(i') } |C_{i^*}|} \\
         &\geq \min_{i^* \in N(i')} \frac{ |C_{i^*,h}|  - \mhout   }{ |C_{i^*}|} \geq l_h
    \end{align*}
Note that the above inequalities follow from the \cref{cauch_ineq} since for any  center $i^*\in N(i')$, $\frac{ |C_{{i^*},h}|  - \mhout   }{ |C_{i^*}|} \geq  l_h $ and $\frac{ |C_{{i^*},h}|  + \mhin   }{ |C_{i^*}|} \leq  u_h.$

Furthermore, note that from the definition of $\phi'$ we know that $|N(i')|\geq 1$ for any $i'$ in $\hat{S}$, i.e., there exists at least one cluster $C_{i^*}$ in the optimal clustering that is assigned to $i'$. 
This concludes that each center $i' \in S$ is indeed \emph{robust fair}, i.e., satisfies the constraints \eqref{eq:ub} and \eqref{eq:lb}. 
\end{proof}
\end{proof}

\begin{restatable}{lemma}{fairnessviolation}\label{lemma:fairness_violation}
     $(\hat{S},\hat{\phi})$ returned by  \cref{alg:lbfaircluster} is a $\lambda$-violating solution where $\lambda $ is at most $\frac{2}{\outm}.$
\end{restatable}
\begin{proof}
According to \cref{eq:robust_fair_viol} a $\lambda$-violating solution is only required to satisfy the following reduced constraints.
\begin{equation*}
 \forall i\in S, h\in \cH :  \frac{|C_{i,h}({\chi})| + \mhin }{ |C_i|} \leq  u_h +\lambda \text{ and } \frac{|C_{i,h}({\chi})|- \mhout }{ |C_i|}  \geq  l_h  -\lambda 
\end{equation*}
Therefore, for any given clustering $\{C_i\}_{i\in S}$, we have the fairness violation $\lambda$ as follows,
\begin{align*}
    \lambda \leq \max_{h \in \cH, i\in S} \left\{ \frac{l_h|{C}_i| - |{C}_{i,h}(\chi)| + \mhout }{|{C}_i|},  \frac{ |{C}_{i,h}(\chi)| + \mhin - u_h|{C_{i}}| }{|{C_i}|}  \right\}
\end{align*}
Before we delve into the proof we note that any LP solution $\mathrm{x}^{\text{LP}}$ that satisfies  constraints \cref{lp:rfa}-\cref{eq:endrfa}, the following holds:
\begin{align}
    l_h |C_i^{\text{LP}}| -  |C_{i,h}^{\text{LP}}| + \mhout \leq 0  \label{eq:first}\\ 
    |C_{i,h}^{\text{LP}}| -u_h |C_i^{\text{LP}}| + \mhin \leq 0 \label{eq:second}
\end{align}
Where Inequalities \eqref{eq:first} and \eqref{eq:second} follow from constraints \eqref{cons:lb_trf} and \eqref{cons:ub_trf}, respecitvely, by simple algebreic manipulation.

We note further by \cref{cih_lowerbound} that $|C_i^{\text{LP}}|$ satisfies:
\begin{align}
    |C_i^{\text{LP}}| \ge \outm \label{eq:LP_lower bound}
\end{align}

Now, let ${C}_i^{\text{Integ}}$ denote the cluster corresponding to the set of points that are assigned to center $i \in S$. The fairness violation from the lower bound can be bounded as follows 
    \begin{align*}
   \frac{l_h|{C}_i^{\text{Integ}}| - |{C}_{i,h}^{\text{Integ}}| + \mhout }{|{C}_i^{\text{Integ}}|}   & \leq  \frac{l_h (|C_i^{\text{LP}} |+ 1) - (|C_{i,h}^{\text{LP}}|-1) + \mhout }{\floor{|C_i^{\text{LP}}|}}  \quad \quad \text{(by the bounds in in \cref{lemma:obj_rounded})}\\ 
   &=  \frac{l_h |C_i^{\text{LP}}| - |C_{i,h}^{\text{LP}}|+\mhout }{\floor{|C_i^{\text{LP}}|}} + \frac{ l_h+1 }{\floor{|C_i^{\text{LP}}|}} \\ & 
   \leq  0 +   \frac{ l_h+1 }{\floor{|C_i^{\text{LP}}|}} \quad \quad \text{(by Inequality \eqref{eq:first})} \\
            & \leq   \frac{ l_h+1 }{\outm} \quad \quad \text{(by Inequality\eqref{eq:LP_lower bound} since $\outm$ is an integer)}
        \end{align*}
        
    
Similarly, we can bound the violation from the upper bound as follows 
\begin{align*}
    \frac{ |{C}_{i,h}^{\text{Integ}}| + \mhin - u_h|{C_{i}}^{\text{Integ}}| }{|{C_i}^{\text{Integ}}|} 
             &\leq  \frac{  (|C_{i,h}^{\text{LP}}|+1) + \mhin -u_h (|C_i^{\text{LP}}| - 1) }{\floor{|C_i^{\text{LP}}|}}  \quad \quad \text{(by the bounds in in \cref{lemma:obj_rounded})}\\
            &=   \frac{  |C_{i,h}^{\text{LP}}| -u_h |C_i^{\text{LP}}| + \mhin  }{\floor{|C_i^{\text{LP}}|}}  + \frac{1+u_h}{\floor{|C_i^{\text{LP}}|}}  \\ 
            & \leq 0 + \frac{ u_h  +1 }{\floor{|C_i^{\text{LP}}|}} \quad \quad \text{(by Inequality \eqref{eq:second})} \\
            & \leq \frac{ u_h  +1 }{\outm }  \quad \quad \text{(by Inequality\eqref{eq:LP_lower bound} since $\outm$ is an integer)}
\end{align*}
Finally, we have the following bound on $\lambda$,
\begin{align*}  
\lambda \leq  {\max_{h \in \cH} \left\{ \frac{ l_h  +1 }{\outm} ,\frac{1+u_h }{\outm} \right\} }    <  \frac{2}{\outm} 
\end{align*}
The last inequality is due to the fact that $l_h \leq u_h<1$. 
\end{proof}

\finaltheorem*
\begin{proof}
For any given instance of \rfc{}, any non-empty solution $(\hat{S},\hat{\phi})$ returned by \textsc{RobustAlg} (\cref{alg:lbfaircluster}) has a cost of at most $3R^*$ guaranteed by the fractional assignment returned by LP $(\hat{S},\hat{\phi})$ as no point is fractionally assigned to any center at a distance more than $3R^*$. Further, part (1) of \cref{lemma:obj_rounded} guarantees that the radius would not increase, therefore the cost would still be at most $3R^*$.
Moreover, from \cref{lemma:fairness_violation} it follows that $(\hat{S},\hat{\phi})$ is a $\lambda$-violating solution with $\lambda  \leq  \frac{2}{\outm}$. This concludes the proof of the theorem.
\end{proof}

\vanillabad*

\begin{proof}
We prove this theorem using a counter example. Specifically, consider the instance in \cref{fig:toy_example} with $n=4$ points belonging to red and blue groups. Let $k=2$ and $l_{\text{red}} = l_{\text{blue}}=1/4$ and $u_{\text{red}} = u_{\text{blue}}=3/4$. The noise parameters are set to $m_{\text{red}}^+=m_{\text{red}}^-=1.$ 
The vanilla $k$-center algorithm selects the centers $\vS = \{s_1, s_2\}$. Consider the lower bound constraints \eqref{cons:lb_trf} in the corresponding LP with $S=\vS$.
This constraint requires that each center $s_1$ and $s_2$ must have a fractional assignment strictly greater than $1$ from each color. This follows from the lower bound constraints in \cref{cons:lb_trf} where each center in $s \in \vS$ and $h \in \{\text{red,blue}\}$ must satisfy
\begin{align*}
    |C_{s,h}^{\text{LP}}| \geq \mhm+l_h |C_{s}^{\text{LP}}| \geq 1 + l_h |C_{s}^{\text{LP}}| >1 
\end{align*}
where $ |C_{s,h}^{\text{LP}}|$ denotes the fractional assignment of strictly greater than one point of color $h$ to center $s$ and $|C_s^{\text{LP}}|$ denotes the fractional assignment of strictly greater than two points to center $s.$ 
However, both $s_1$ and $s_2$ cannot simultaneously be fractionally assigned strictly greater than $2$ points each since there are only $4$ points in total. Therefore, this shows that there exists no feasible solution to LP~\eqref{lp:rfa}-\eqref{eq:endrfa} for any arbitrarily large $R.$
\end{proof}
\begin{figure}[H]
    \centering
    \includegraphics[scale= 0.5]{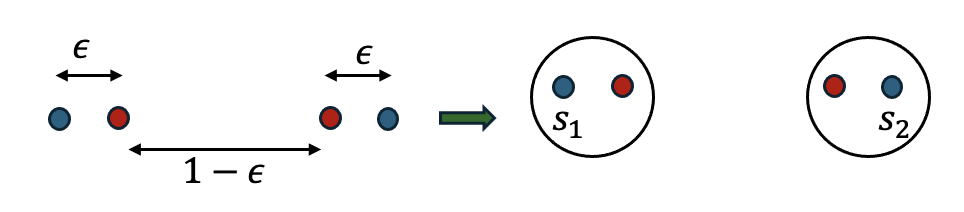}
    \caption{Toy example showing that for the set of centers $\vS$ from vanilla $k$-center algorithm, there does not exist a feasible solution to LP~\eqref{lp:rfa}-\eqref{eq:endrfa} for any arbitrarily large $R$.}
    \label{fig:toy_example}
\end{figure}

\input{maxFlow}

\section{More Discussion About Previous Noise Models and Their Algorithms in Fair Clustering}\label{app:prior_noise}
\subsection{Drawbacks of the Noise Model of Probabilistic Fair Clustering \cite{esmaeili2020probabilistic}}\label{app:pfc_discussion}
We provide an example to show the drawbacks of the probabilistic model introduced in \citet{esmaeili2020probabilistic}. Their theoretical guarantees for robust fair clustering only guarantee to satisfy the fairness constraints in expectation. However, the realizations can be arbitrarily unfair. 
We illustrate this using an example.
Consider a set of $14$ points as shown in the \cref{fig:negative_example} where each point is assigned to either the red or blue group, each with a probability of $1/2$. Any clustering of these points is fair in expectation assuming the lower and upper bounds for both the blue and red group are close to $\frac{1}{2}$. However individual realizations can be unfair. Specifically, since the joint probability distribution is not known (in fact, it is not incorporated at all in the probabilistic model of \citet{esmaeili2020probabilistic}) the realizations could be as shown in the figure where all points in a cluster take on the same color simultaneously in a realization. This shows that the probabilistic model may return clusters that can be fair (proportional) in expectation but completely unfair (unproportional) in realization.    


\begin{figure}[H]
    \centering
    \includegraphics[width=1.1\textwidth]{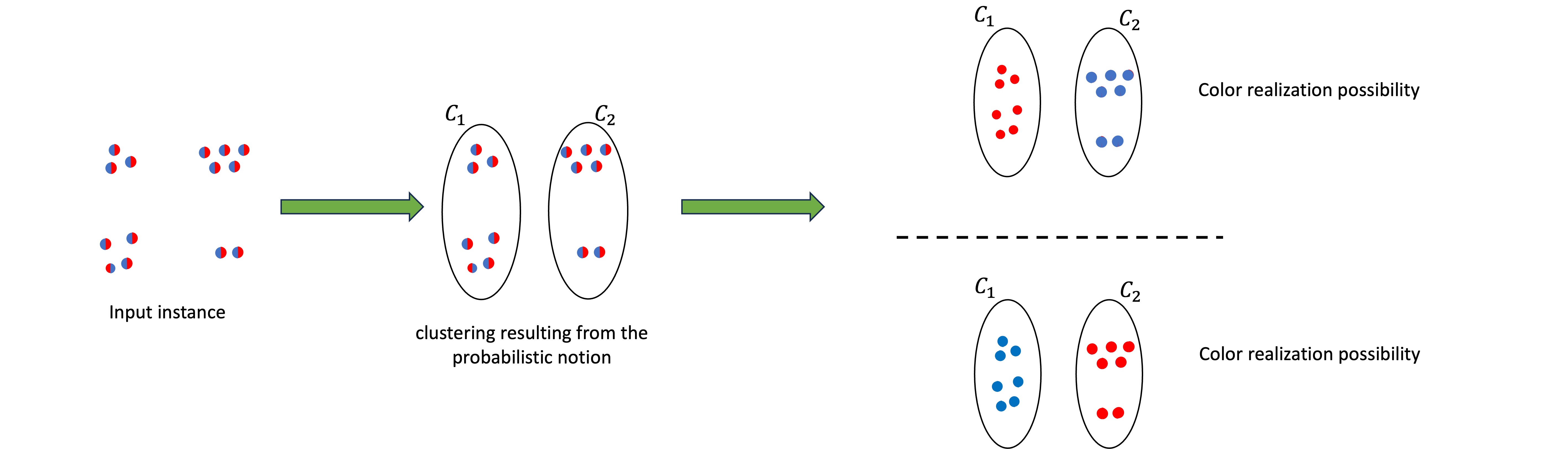}
    \caption{An instance of \fc{} with $14$ points where each point has the probabilities $p_{j}^{\text{red}}=p_{j}^{\text{blue}}=\frac{1}{2}$. Here $p_{j}^{\text{red}}$ and $p_{j}^{\text{blue}} \forall j \in \cP$ denote the probability that point $j$ belong to red and blue groups respectively. Note how the probabilistic fair clustering satisfies fairness in expectation. However, the two realizations shown demonstrate that color proportionality can be completely violated. }
    \label{fig:negative_example}
\end{figure}

\subsection{More Discussion About the Noise Model of \cite{Chhabra23:Robust} and Their Algorithm}\label{app:chhabra_discussion}
We provide an example to show the drawbacks of the noise model introduced by \citet{Chhabra23:Robust}. Their noise model assumes that only a subset of the points are affected by the adversary but they do not specify how one can access this subset. In their experiments, they generate this subset using random sampling. Specifically, they independently sample each point with probability $0.15$ to obtain a subset of points $\cP' \subseteq \cP.$ However, we can easily construct examples where their random sampling method with probability $\approx 0.99$ can never return some subsets. We illustrate this using an example. Consider an instance of \fc{} as shown in \cref{fig:negative_example_two} where only a subset of points have incorrect memberships according to \citet{Chhabra23:Robust}. Their sampling process selects a subset (comprising 10\% of the points) by randomly sampling each point independently with a probability of $0.1$. As a result, out of $160$ points $16$ points are perturbed in expectation. However, this does not capture scenarios where all perturbations occur within a subset of points (as shown in the left side of \cref{fig:negative_example_two}) as the probability of such an event is close to $0$, precisely $0.0002$. On the other hand, our model considers for all possible subsets with $16$ points.  As a result, we can model the scenarios where all the incorrect memberships occur in a single group or more generally all possible combinations across the two groups. 

\begin{figure}[H]
    \centering
    \includegraphics[scale=0.4]{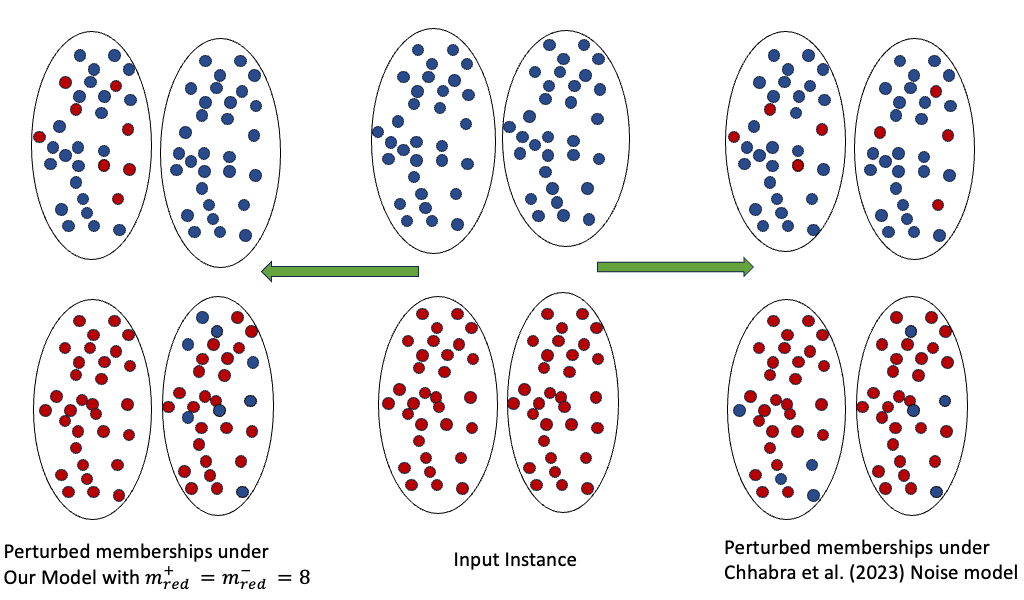}
    \caption{Consider an instance of \fc{} with $n=160$ points, where $10\%$ of the points have incorrect memberships. \citet{Chhabra23:Robust} perturbs a subset (comprising 10\% of the points) by randomly sampling each point independently with a probability of $0.1$. As a result $16$ points are perturbed in expectation. However, this does not capture scenarios where all perturbations occur within a subset of points as shown on the left side. This is because the probability of such an event is $\approx 0$. On the other hand, our model considers for all possible subsets with $16$ points, thereby covering all possible scenarios.}
    \label{fig:negative_example_two}
\end{figure}

Furthermore, in terms of algorithms \citep{Chhabra23:Robust} also provide a defense algorithm by a Consensus Clustering method via k-means (Lloyd's algorithm) combined with fair constraints to achieve robustness against their proposed attack. Their algorithm trains a neural network using a loss function based on pair-wise similarity information obtained by running Consensus $k$-means clustering, which is different from our $k$-center objective. Moreover, their \emph{fair clustering loss} fails to include the fractional proportional bounds present in our fair clustering instances. Therefore, their algorithm is inapplicable to our problem.
Further, their algorithm provide no theoretical or empirical guarantees on the distance-based clustering cost. Finally, their fairness loss does not model the fairness constraints based on proportional lower and upper bounds or any other parameters provided by the stakeholder. 
More importantly, their robust algorithms have no theoretical guarantees on ex-post fairness violations, whereas ours do.

\section{Additional Details on our Experimental Setup \& Experimental Results}\label{app:experiments}
In Section~\ref{app:experiments-setup} we expand on the libraries and hardware used to complete experiments.
In Section~\ref{sec:exp-running-times} we discuss the running times of our algorithm \rfcalg{} and baselines presented in Figure~\ref{fig:m-and-T-plots} of Section~\ref{sec:experiments}.
In Section~\ref{sec:diabetes-experiments} we use a fourth dataset, \diabetes, to compare the algorithms' running times as the number of points $n$ increases.
Lastly, in Section~\ref{sec:smaller-m-experiments} we repeat the experiments in Figure~\ref{fig:m-and-T-plots} of Section~\ref{sec:experiments} for a different range of $m$ values. 



\subsection{Experimental Setup}\label{app:experiments-setup}
The experiments are run on Python 3.6.15 on a commodity laptop with a Ryzen 7 5800U and 16GB of RAM.
The linear programs (LP) in the algorithms are solved using CPLEX~12.8.0.0 \citep{nickel2022ibm} and flow problems are solved using \texttt{NetworkX}~2.5.1 \citep{hagberg2013networkx}.
In total, the experiments in Figure~\ref{fig:m-and-T-plots} solve $53$ fair clustering instances (30 robust fair instances, 20 probabilistic fair, and 3 deterministic fair).
\probalg{} is not run on \cens because its theoretical guarantees hold for the two-color setting only.
\detalg{} has a single fair clustering instance per dataset because \detalg{} does not depend on $m$.
Our code implementation forks the code of \cite{dickerson2023doubly}.

\subsection{Running Times of the Experiments} \label{sec:exp-running-times}
Figure~\ref{fig:runtime-large-m} shows \rfcalg{} outperforms the baselines in both of the larger datasets (\adult and \cens); this is not the case in \bank, but the difference is at most 20 seconds there.
\begin{figure}
    \centering
    \includegraphics[width=\textwidth]{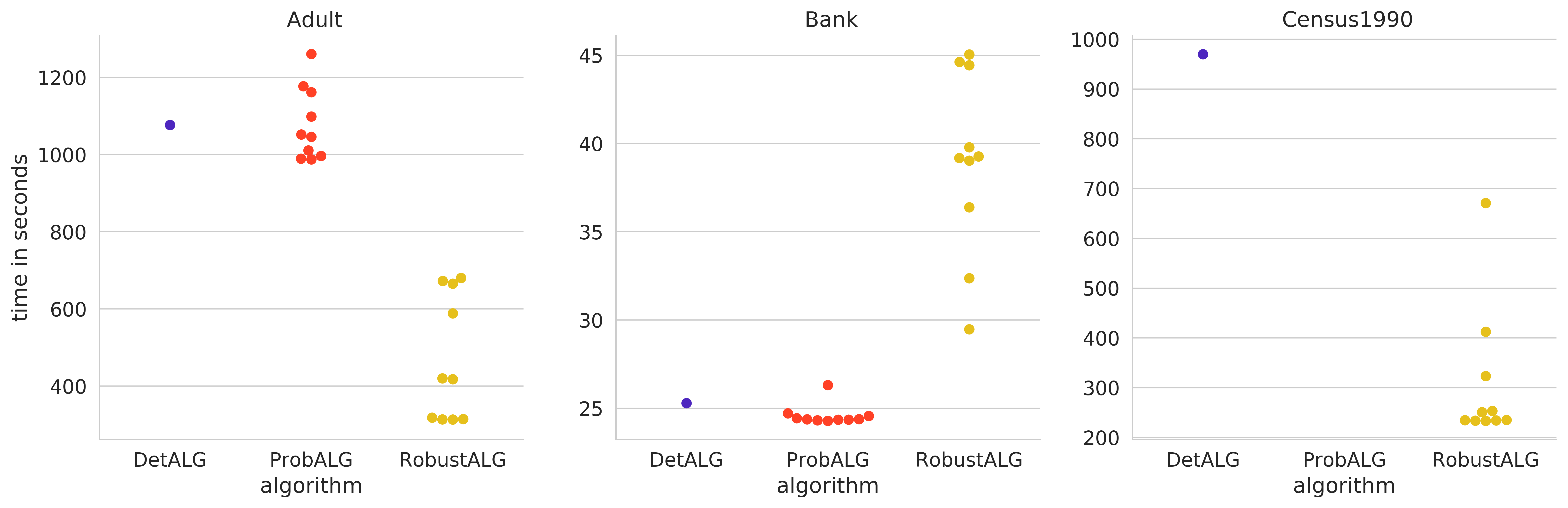}
    \caption{
    Wall clock time each algorithm took to solve the associated \{deterministic, probabilistic, robust\} fair clustering instances in Figure~\ref{fig:m-and-T-plots}.}
    \label{fig:runtime-large-m}
\end{figure}
In each step of the binary search, \rfcalg{} has an additional step of \textsc{GetCenters} which contributes to the running time.
However, the running time of these algorithms is largely dominated by solving the associated LPs.
However, the LPs for the baselines always use $k$ centers and thus have $nk$ variables.
However, \rfcalg{} uses the $k' \le k$ centers selected by \textsc{GetCenters}.
In practice, this can be a substantial speedup as the difference between $k'$ and $k$ increases and as $n$ increases.
This behavior is also observed in \ref{sec:diabetes-experiments}.

Related to this phenomenon, as $m$ grows, the binary search in \rfcalg{} discards smaller values of $R$ because more LPs become infeasible.
Therefore, the binary search moves onto greater values of $R$, for which \textsc{GetCenters} selects fewer centers.
Indeed, in Figure~\ref{fig:runtime-large-m-vs-m} we can observe \rfcalg{} speeding up as $m$ increases.
On the other hand, and as expected, the running time of \probalg{} does not exhibit this interaction with $m$.

\begin{figure}
    \centering
    \includegraphics[width=\textwidth]{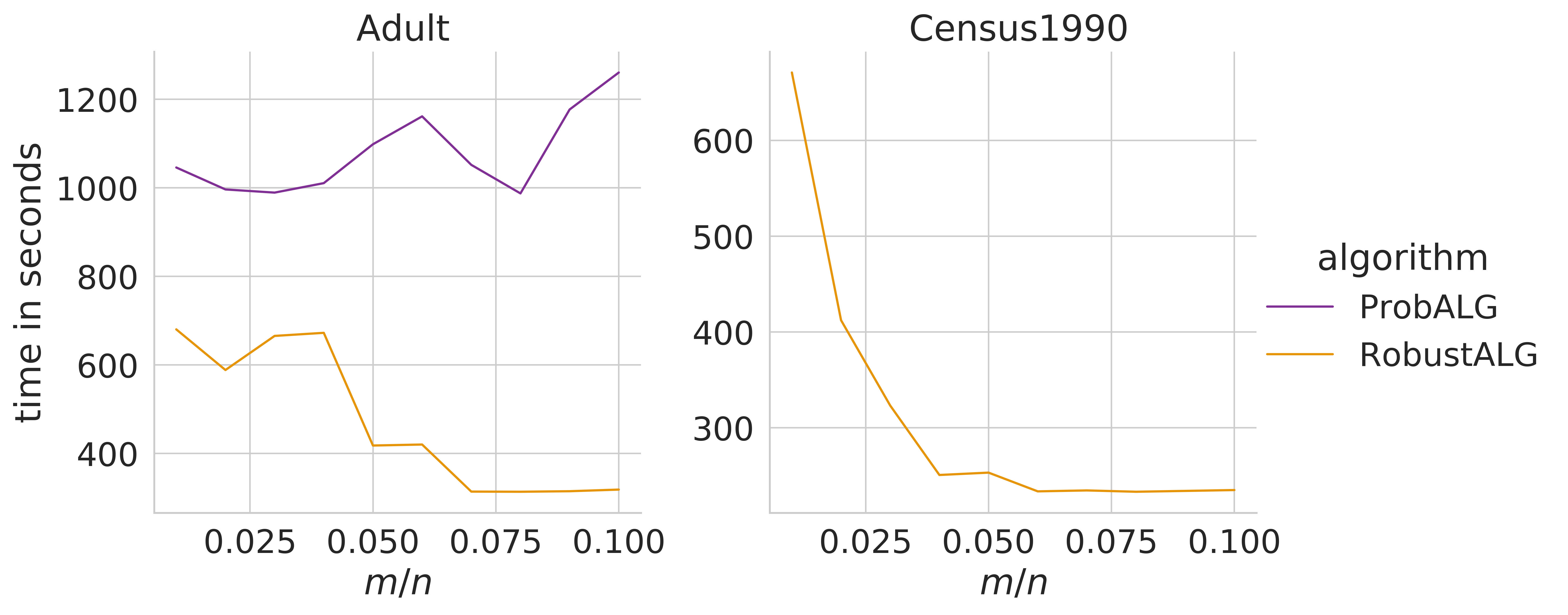}
    \caption{Running times of \probalg{} and \rfcalg{} from Figure~\ref{fig:runtime-large-m} but plotted against $m$.}
    \label{fig:runtime-large-m-vs-m}
\end{figure}

\subsection{Additional Experiments on the \diabetes dataset} \label{sec:diabetes-experiments}
We use an additional dataset \diabetes to supplement our running time experiments.
Like our other three datasets, \diabetes is also from the UCI repository.
It has 49 features and we set up two colors corresponding to whether the patient (i.e., point) was or was not female.
Figure~\ref{fig:runtime-diabetes} shows the dependence of all three algorithms on $k$, and we can also see that the baselines are more sensitive to larger $k$.
These plots use the first $n'$ points of \diabetes for $n'$ in $\{2000, 4000, 6000, \dots, 23000\}$.
For each instance we take $m = 0.005n'$.
The proportionality constants $u_h$ and $l_h$ are set to be feasible for \rfcalg{} exactly in the same way as in Section~\ref{sec:experiments}.

We also recreate the fairness violation and objective plots from Figure~\ref{fig:m-and-T-plots} on \diabetes as shown in Figure \ref{fig:diabetes-ordinary}. For these experiments we fix $k=5$ and $\forall h$, $l_h = 0.3$ and $u_h=0.75$. A moment's reflection shows the worst fairness violation possible is 0.3. Indeed, \detalg{} and \probalg{}, respectively, exactly reach and get very close to this violation.
The objectives also only differ by at most a meager 4 units of distance.
\begin{figure}[h]
    \centering
    \includegraphics[width=\textwidth]{Figures/experiments/diabetes_runtimes.png}
    \caption{Running times of \rfcalg{}, \probalg{}, and \detalg{} on dataset \diabetes as the number of points $n$ increases. We run the experiments for number of clusters $k$ of 5 and 10.}
    \label{fig:runtime-diabetes}
\end{figure}

\begin{figure}[h]
    \centering
    \includegraphics[width=0.8\textwidth]{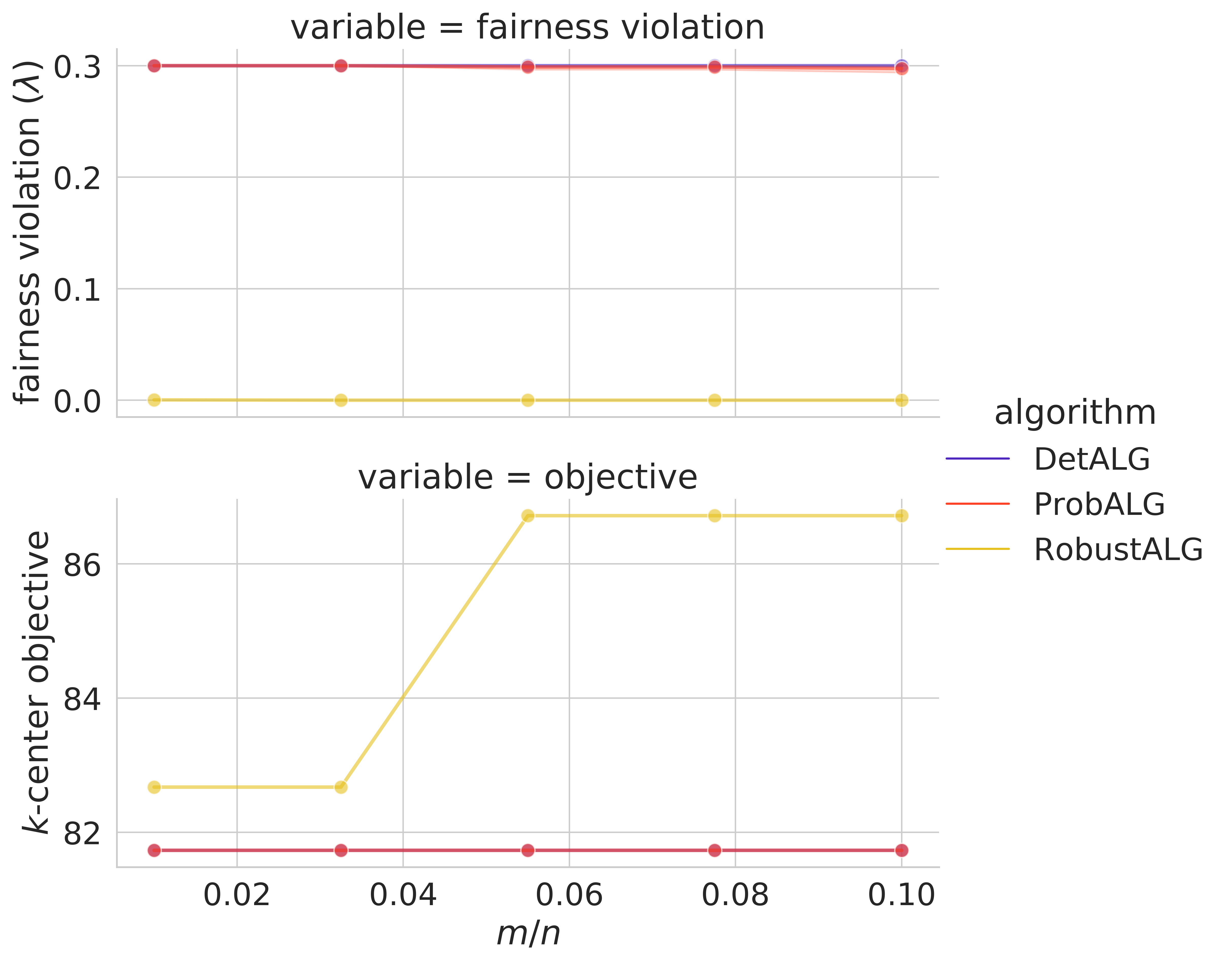}
    \caption{Plots of $k$-center objective and fairness violations on \diabetes.}
    \label{fig:diabetes-ordinary}
\end{figure}

\subsection{Additional experiments on smaller values of $m$} \label{sec:smaller-m-experiments}
Figure~\ref{fig:small-m-plots} recreates the experiments in Figure~\ref{fig:m-and-T-plots} but with values of $m$ one order of magnitude smaller and $k=5$.
These plots show greater variance in the fairness violations of \probalg{}, which makes sense given the smaller $m$ and that, in probabilistic fair clustering, the probability a point's label is correct is $p_{\text{acc}} = 1 - m/n$.
\begin{figure}[h]
    \centering
    \includegraphics[width=\textwidth]{Figures/experiments/small_m_m_vs_fairvio.png}
    \caption{A repeat of Figure~\ref{fig:m-and-T-plots} but for smaller values of $m$ and $k=5$ instead of $k=10$.}
    \label{fig:small-m-plots}
\end{figure}

%% file: maxFlow.tex
\section{\textsc{MaxFlow} Rounding}\label{sec:maxflow}
\textsc{MaxFlow} rounding is given an LP solution satisfying \eqref{lp:rfa},\eqref{cons:ub_trf},\eqref{cons:lb_trf}, and \eqref{eq:endrfa}. The \textsc{MaxFlow} rounding we use is identical to the one used in \cite{dickerson2023doubly,ahmadian2019clustering,bercea2018cost}. Our explanation here closely follows that in \cite{dickerson2023doubly} with some modifications for our setting. Formally, given an LP solution $\mathrm{x}^{\text{LP}}=\{x^{\text{LP}}_{i,j}\}_{i\in \sgood, j \in \cP}$, the network flow diagram is constructed as follows:
\begin{enumerate}
    \item $V= \{s,t\} \cup \cP \cup \{i^h| i \in \sgood, h \in \cH\} \cup \{i \in \sgood\}$. 
    \item $A = A_1 \cup A_2 \cup A_3 \cup A_4$ where $A_1=\{(s,j)| j \in \cP\}$ with upper bound of 1. $A_2=\{(j,i^h) |   x^{\text{LP}}_{i,j} >0\}$ with upper bound of 1. The arc set $A_3=\{(i^h,i)| i \in \sgood, h \in \cH \}$ with lower bound $\floor{\sum_{j \in \cP_h}x^{\text{LP}}_{i,j}}$ and upper bound of $\ceil{\sum_{j \in \cP_h}x^{\text{LP}}_{i,j}}$. As for $A_4=\{(i,t)| i \in \sgood\}$ the lower and upper bounds are $\floor{\sum_{j \in \cP}x^{\text{LP}}_{i,j}}$ and $\ceil{\sum_{j \in \cP}x^{\text{LP}}_{i,j}}$. 
\end{enumerate}
By construction of the network flow diagram the maximum flow that can be achieved at the sink $t$ is $n$ (the number of points). Further, the given LP assignment $\mathrm{x}^{\text{LP}}=\{x^{\text{LP}}_{i,j}\}_{i\in \sgood, j \in \cP}$ is a valid fractional flow that achieves a maximum flow of $n$. Since the upper and lower bound on the arcs are integral, it follows by standard result of the max flow problem that we can find an integral maximum flow $\mathrm{x}^{\text{Integ}}=\{x^{\text{Integ}}_{i,j}\}_{i\in \sgood, j \in \cP}$. Moreover, from the set upper and lower bounds the following is immediate: 
\begin{align*}
      \floor{\sum_{j\in \cP}\mathrm{x}_{i,j}^{\text{LP}}} \leq & \sum_{j\in \cP}\mathrm{x}_{i,j}^{\text{Integ}} \leq  \ceil{\sum_{j\in \cP}\mathrm{x}_{i,j}^{\text{LP}}}  \\ 
       \floor{\sum_{j\in \cP_h}\mathrm{x}_{i,j}^{\text{LP}}} \leq & \sum_{j\in \cP_h}\mathrm{x}_{i,j}^{\text{Integ}} \leq  \ceil{\sum_{j\in \cP_h}\mathrm{x}_{i,j}^{\text{LP}}} 
\end{align*}
Further, recall from \cref{lemma:obj_rounded} that $|C_i^{\text{LP}}|= \sum_{j\in \cP}\mathrm{x}_{i,j}^{\text{LP}}$, $|C_i^{\text{Integ}}|= \sum_{j\in \cP}\mathrm{x}_{i,j}^{\text{Integ}}$,  $|C_{i,h}^{\text{LP}}|=\sum_{j\in \cP_h}\mathrm{x}_{i,j}^{\text{LP}}$, and  $|C_{i,h}^{\text{Integ}}|=\sum_{j\in \cP_h}\mathrm{x}_{i,j}^{\text{Integ}}$. Therefore, it follows that we have 
\begin{align*}
       \floor{|C_i^{\text{LP}}|} \leq & |C_{i}^{\text{Integ}} | \leq  \ceil{|C_i^{\text{LP}}|},\\ 
       \floor{|C_{i,h}^{\text{LP}}|} \leq &|C_{i,h}^{\text{Integ}}|\leq \ceil{|C_{i,h}^{\text{LP}}|}  
\end{align*}

This proves part (2) of \cref{lemma:obj_rounded}. Part (1) of \cref{lemma:obj_rounded} simply follows from that fact that if $\mathrm{x}^{\text{LP}}_{i,j}=0$ then there would not be an arc connecting point (vertex) $j$ to vertex $i^h, \forall h \in \cH$ and that immediately implies that $\mathrm{x}^{\text{Integ}}_{i,j}=0$.